\documentclass[lettersize,journal]{IEEEtran}
\usepackage[utf8]{inputenc}
\usepackage{amsmath,amsfonts}
\usepackage{algorithmic}
\usepackage{algorithm}
\usepackage{array}
\usepackage[caption=false,font=normalsize,labelfont=sf,textfont=sf]{subfig}
\usepackage{textcomp}
\usepackage{stfloats}
\usepackage{url}
\usepackage{verbatim}
\usepackage{graphicx}
\usepackage{cite}
\usepackage{tikz}
\usetikzlibrary{positioning, arrows.meta}
\usepackage{amsmath,amssymb,amsthm}
\usepackage[colorlinks=true, linkcolor=blue, citecolor=blue, filecolor=blue, urlcolor=blue]{hyperref}
\usepackage{xcolor}
\usepackage{epstopdf} 
\usepackage{tcolorbox}
\usepackage{colortbl}
\usepackage{multirow}

\newtheorem{theorem}{Theorem}[section]

\begin{document}

\title{Latent Danger Zone: Distilling Unified Attention for \\ Cross-Architecture Black-box Attacks}
% \title{Latent Diffusion Meets Joint Attention: Unifying CNN and Transformer Weaknesses for Black-box Adversarial Attacks}

\author{Yang Li,~\IEEEmembership{Member,~IEEE}
        Chenyu Wang,
        Tingrui Wang, 
        Yongwei Wang,~\IEEEmembership{Member,~IEEE} \\
        Haonan Li,
        Zhunga Liu,~\IEEEmembership{Member,~IEEE} 
        Quan Pan,~\IEEEmembership{Member,~IEEE} 
        
\thanks{This work was partly by the Key Program of the National Natural Science Foundation of China under Grant 62233014, and partly supported by the Youth Program of the National Natural Science Foundation of China under Grant 62103330. \textit{(Corresponding author: Yang Li.)}}
 % <-this % stops a space
\thanks{Yang Li, Chenyu Wang, Tingrui Wang, Haonan Li, Zhunga Liu and Quan Pan are with the School of Automation at Northwestern Polytechnical University, Xi’an, 710129, China (e-mail: liyangnpu@nwpu.edu.cn).}
\thanks{Yongwei Wang is with the School of Zhejiang University, Hangzhou, 310007, China.}}

\makeatother

\maketitle

\begin{abstract}
Black-box adversarial attacks remain challenging due to limited access to model internals. Existing methods often depend on specific network architectures or require numerous queries, resulting in limited cross-architecture transferability and high query costs. To address these limitations, we propose JAD, a latent diffusion model framework for black-box adversarial attacks. JAD generates adversarial examples by leveraging a latent diffusion model guided by attention maps distilled from both a convolutional neural network (CNN) and a Vision Transformer (ViT) models. By focusing on image regions that are commonly sensitive across architectures, this approach crafts adversarial perturbations that transfer effectively between different model types. This joint attention distillation strategy enables JAD to be architecture-agnostic, achieving superior attack generalization across diverse models. Moreover, the generative nature of the diffusion framework yields high adversarial sample generation efficiency by reducing reliance on iterative queries. Experiments demonstrate that JAD offers improved attack generalization, generation efficiency, and cross-architecture transferability compared to existing methods, providing a promising and effective paradigm for black-box adversarial attacks.
\end{abstract}

\section{Introduction}
\IEEEPARstart {A}{dversarial} attacks against deep neural networks have achieved remarkable results in controlled settings. Yet, black-box attack scenarios — where an attacker has no access to model internals — remain challenging. A common approach to black-box attacks is to craft adversarial examples on a surrogate model (often a convolutional neural network, CNN)\cite{he2016deep,Liu_2022_CVPR,DBLP:journals/corr/RenHG015,DBLP:journals/corr/RonnebergerFB15} and rely on their transferability to fool the unseen target model. However, this transfer-based strategy typically assumes the target shares a similar architecture to the surrogate. In practice, adversarial perturbations optimized on a specific architecture often fail to transfer across fundamentally different model families~\cite{10378472}. For instance, an adversarial pattern generated for a CNN may exploit texture biases that a Vision Transformer (ViT)\cite{dosovitskiy2021imageworth16x16words,DBLP:journals/corr/abs-2104-14294,elnouby2021xcitcrosscovarianceimagetransformers,carion2020endtoendobjectdetectiontransformers} does not share, yielding significantly degraded attack success on the ViT. In summary, current black-box attacks largely leverage architecture-specific features (e.g., CNN textures), making it difficult to attack models of other architectures, such as Transformers effectively.

Recently, generative approaches have been proposed to improve black-box attack efficiency. Instead of performing many queries or iterative optimization for each new input, these methods learn a generator that produces adversarial examples directly. For example, the Conditional Diffusion Model Attack (CDMA)\cite{liu2024boosting} trains a diffusion model to transform any benign input into an adversarial example in one forward pass. Such a generative attacker can drastically cut down query counts, in some cases to nearly one query per attack, making black-box attacks far more practical. However, the CDMA study does not address cross-architecture attacks against Transformer-based models. Our experiments reveal that CDMA’s performance degrades when attacking target networks whose architectures differ from the surrogate used during training, highlighting its limited cross-architecture generalization. This suggests that existing generative approaches implicitly overfit to architecture-specific patterns, such as the texture bias in CNNs, and struggle to capture vulnerabilities shared across heterogeneous model families.
To improve transferability between ViTs and CNNs, the Integrated Gradients-based method creates transferable adversarial examples by using path-integrated gradients and momentum to escape local optima~\cite{10378472,ren2025improving}. This method involves iterative gradient optimization for both ViTs and CNNs. However, it relies on iterative computation of integrated gradients and momentum-based accumulation, which depends on source model gradients and multiple iterations to estimate optimal perturbations.

In military operations, a sniper and an observer form a deadly duo. As illustrated in Fig. \ref{fig: military}, the sniper uses a precision rifle to target and strike specific objectives, while the spotter has a broader vantage point, identifying strategically critical areas. Individually, they excel at different tasks — the sniper at pinpoint accuracy, the observer at situational awareness. However, it is their coordination that amplifies their effectiveness, allowing them to eliminate targets with unmatched precision and range jointly. Similarly, CNNs and Vision Transformers possess distinct strengths: CNNs excel at capturing local textures and edges, analogous to the sniper's focus, while ViTs provide a broader perspective by establishing global dependencies across patches, akin to the observer's strategic oversight \cite{raghu2022visiontransformerslikeconvolutional}. Yet, existing generative attacks, like CDMA, function as a lone sniper — highly effective within a single architecture but unable to leverage the strategic advantage offered by joint attention from both CNNs and ViTs. Thus, a natural question arises: Can we design a generative attack framework that coordinates the strengths of both CNNs and ViTs to strike vulnerable regions that are commonly sensitive across both architectures?

Motivated by this analogy, we propose {\bf{JAD}} – a Joint Attention Distillation in the latent space Attack for black-box adversarial example generation. Our approach employs a latent diffusion model as a universal adversarial generator guided by knowledge distilled from both CNN and ViT models. In a one-time offline training phase, we fuse the attention maps (saliency indicators of model sensitivity) from a CNN surrogate (e.g., VGG) and a ViT surrogate into a unified attention objective. By distilling the joint attention of these heterogeneous models, we teach the diffusion generator where both model types are vulnerable. In other words, JAD aligns the perturbation focus on the common fragile regions shared by CNNs and Transformers. This attention-guided training drives the latent diffusion model to produce adversarial perturbations that simultaneously exploit features important to both architectures. As a result, the generated attacks exhibit strong cross-architecture transferability, succeeding against a wide range of model families in the black-box setting. Importantly, because JAD yields a direct generative mapping from inputs to adversarials, it retains the query efficiency advantage, often requiring only a verification query or none at all during deployment, similar to CDMA’s paradigm.

We conduct extensive experiments on three standard vision benchmarks (ImageNet, CIFAR-10, and CIFAR-100), comparing JAD against state-of-the-art black-box attack methods across all three major categories: pure black-box, query-based, and transfer-based. The results demonstrate that JAD consistently achieves higher attack success rates, better cross-architecture generalization, and significantly improved query efficiency relative to these methods. Notably, JAD significantly improves black-box success rates on transformer-based vision models without needing architecture-specific tuning, while also reducing the query cost per attack. These advantages underscore the effectiveness of JAD across diverse black-box attack scenarios. 
% Figure 1
\begin{figure}[!tb]
	\centering
	\includegraphics[width=1\linewidth]{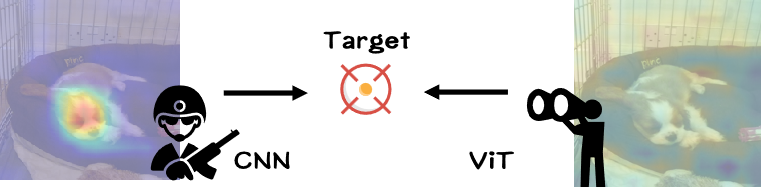}
	\caption{Illustration of the sniper-spotter analogy for CNN and ViT models: Similar to a sniper, the CNN model focuses on fine-grained local details (top heatmap), targeting specific regions with high precision. Meanwhile, the ViT model, analogous to a spotter, provides a broader, global perspective (bottom heatmap), capturing contextual semantic information across the entire scene. The combination of these complementary views helps to identify and target the most vulnerable regions effectively.}
    \label{fig: military}
    \vspace{-0.4cm}
\end{figure}
In summary, our contributions are as follows:

\begin{itemize}
\item We identify and address the challenge of attacking fundamentally different model architectures in black-box settings. JAD is the first generative attack framework to bridge CNN and Transformer vulnerabilities, enabling effective transferable attacks across architectures.
\item We introduce a novel attention distillation technique that fuses saliency maps from a CNN and a ViT into the training of a latent diffusion generator. This guides the generator to align perturbations with critical regions common to both architectures, greatly enhancing attack generalization.
\item By focusing on shared weak spots and leveraging a diffusion model, JAD achieves higher success rates on both CNN and ViT targets than existing methods (e.g., CDMA) under stringent query limits. Our approach demonstrates state-of-the-art transferability across heterogeneous models while often requiring only a single query at attack time, marking a substantial step toward practical and architecture-agnostic black-box adversarial attacks.
\end{itemize}

These results underline the potential of combining insights from disparate network architectures within a generative attack paradigm. By jointly distilling CNN and Transformer perspectives into adversarial example generation, JAD opens a new avenue for highly transferable and query-efficient black-box attacks, moving closer to real-world applicability.

\section{Related Works}
\subsection{Black-box Adversarial Attacks}
Black-box adversarial attacks assume no access to internal model parameters or gradients, relying instead on queries or surrogate models to craft adversarial perturbations~\cite{li2023few}. Existing black-box attack methodologies can be broadly categorized into query-based, transfer-based, and generative approaches. Query-based attacks, such as Zeroth-Order Optimization (ZOO) \cite{Chen_2017}, NES \cite{DBLP:journals/corr/abs-1804-08598}, and boundary-based methods \cite{brendel2018decisionbasedadversarialattacksreliable}, iteratively refine perturbations based on model output confidence scores or decisions. Although effective, these methods typically incur substantial query budgets, limiting their practical applicability. Alternatively, transfer-based approaches exploit adversarial examples crafted on white-box surrogate models to attack unknown target models without explicit queries. Techniques such as momentum-based iterative methods (MI-FGSM) \cite{DBLP:journals/corr/abs-1710-06081}, diverse input transformations \cite{DBLP:journals/corr/abs-1803-06978}\cite{long2022frequencydomainmodelaugmentation}, feature-space attacks \cite{finlay2019improvedrobustnessadversarialexamples}, and ensemble-based strategies \cite{DBLP:journals/corr/LiuCLS16}\cite{DBLP:journals/corr/abs-2103-15571} significantly enhance cross-model transferability. However, the success of these methods heavily depends on the similarity between surrogate and target model architectures, often leading to degraded performance when attacking fundamentally different models (e.g., CNNs vs. Transformers). Recently, generative attack methods leveraging GANs \cite{DBLP:journals/corr/abs-1801-02610}\cite{DBLP:journals/corr/abs-1905-11736}, autoencoders \cite{DBLP:journals/corr/abs-1712-02328}, and diffusion models\cite{liu2024boosting} have emerged to produce adversarial examples with improved perceptual quality and efficiency. For instance, DiffAttack \cite{chen2023diffusionmodelsimperceptibletransferable} introduces perturbations within the latent space of diffusion models, achieving high imperceptibility along with notable cross-architecture transferability. Nevertheless, existing generative attacks typically remain limited in effectively bridging distinct architecture gaps, particularly between CNN and Transformer models, due to their reliance on architecture-specific biases. Addressing this limitation motivates the development of approaches capable of explicitly distilling and exploiting shared vulnerabilities across diverse model architectures, a gap our work seeks to fill.
\subsection{Transferable Attacks}
Enhancing the transferability of adversarial examples across diverse models remains a central challenge in adversarial attack research. Previous works have explored various strategies to ensure that perturbations generated on surrogate models do not overfit specific architectures, thus improving their generalization. A prominent line of research stabilizes gradient computation or modifies optimization routines to avoid local optima. For example, Dong et al.\cite{DBLP:journals/corr/abs-1710-06081} introduced momentum-based iterative methods such as MI-FGSM\cite{dong2018boosting}, significantly enhancing the robustness and transferability of attacks across multiple black-box targets. Lin et al. \cite{Long_Tao_LI_Lei_Zhang_2024} further improved upon this strategy by refining gradient computation positions, enabling the adversarial examples to generalize more effectively. Similarly, Wang et al.~\cite{DBLP:journals/corr/abs-2103-15571} proposed variance tuning strategies to adjust gradient directions dynamically, reducing the risk of convergence to suboptimal, model-specific solutions.

Another important direction involves leveraging data augmentations and model ensemble strategies. Techniques introduced by Xie et al.~\cite{DBLP:journals/corr/abs-1803-06978}, such as input transformations including random cropping and resizing, effectively diversify the perturbations, resulting in better transferability. Long et al.~\cite{long2022frequencydomainmodelaugmentation} demonstrated that integrating gradients from multiple surrogate models can produce universal perturbations with enhanced cross-model generalization. More recently, adaptive model ensemble methods dynamically weight different surrogates during the attack generation process, further boosting transferability\cite{chen2023adaptivemodelensembleadversarial}.

Beyond pixel-level perturbations, attacking in the intermediate feature space or leveraging attention mechanisms has proven effective. Inkawhich et al.~\cite{8953700} demonstrated that perturbations applied in feature spaces yield adversarial examples with improved semantic coherence and cross-model effectiveness. Additionally, Wu et al.\cite{9156604} employed attention mechanisms to concentrate perturbations on model-critical regions, significantly strengthening attack transferability.

The recent emergence of diffusion models has provided a powerful generative framework for adversarial attacks. Chen et al. \cite{chen2023diffusionmodelsimperceptibletransferable} proposed DiffAttack, a method that introduces perturbations in the latent space of diffusion models, achieving high imperceptibility and excellent transferability across various black-box models. By leveraging the expressive capability of diffusion-based generative models alongside guidance from sensitive regions identified by discriminative models, DiffAttack demonstrates robust performance in generating transferable and visually natural adversarial examples.

Despite these advances, current methods still encounter critical challenges, such as poor transferability across fundamentally different architectures like CNNs and Vision Transformers (ViTs), and persistent perceptibility issues of adversarial examples. Therefore, developing a method capable of bridging the architectural gap, generating attacks transferable between CNNs and Transformers while maintaining high perceptual quality, remains an open and significant research problem.
\subsection{Latent Diffusion Models and Attention Distillation}
Latent Diffusion Models (LDMs) have emerged as a powerful class of generative models that perform diffusion in a lower-dimensional latent space (produced by an autoencoder), rather than in the pixel space~\cite{Robin2021}. This approach retains the quality of generated images while drastically improving efficiency, since the diffusion process operates on a compressed representation. Rombach et al. \cite{Robin2021} showed that LDMs can achieve high-fidelity image synthesis with far less computational cost than pixel-space diffusion, reaching an optimal trade-off between complexity and detail preservation. Such models also readily incorporate conditioning mechanisms (e.g. cross-attention for text-to-image), making them flexible and suited for tasks like controlled image generation. Separately, attention distillation (or attention transfer) refers to techniques from knowledge distillation that transfer the “attention” knowledge from one network to another. Zagoruyko et al. \cite{zagoruyko2016paying} introduced this idea by having a student network learn to mimic the spatial attention maps of a teacher network, which improved the student’s performance. Since then, attention-based distillation has been applied in model compression and multi-model learning, even across different network architectures, to align feature importance between a teacher and student. JAD brings these two concepts together in a novel way. It utilizes a latent diffusion model as the engine to generate adversarial examples, benefiting from the efficiency and high capacity of LDMs to explore the image perturbation space. Concurrently, during the training of this diffusion-based generator, JAD distills attention information from multiple reference models (e.g., a set of diverse architectures including CNNs and ViTs) into the generative process. This cross-architecture attention distillation guides the latent diffusion model to produce perturbations that are effective across different models by focusing on universally “important” features. This approach is a key innovation distinguishing JAD from previous diffusion-based attacks like CDMA – instead of treating each model in isolation, JAD’s generator internalizes a broader notion of model-agnostic adversarial focus. By merging latent-space diffusion with attention transfer learning, JAD achieves a synergy that allows it to create highly transferable and query-efficient adversarial examples beyond the reach of prior black-box attack methods.

% Figure 2
\begin{figure*}[!tb]
	\centering
	\includegraphics[width=1\linewidth]{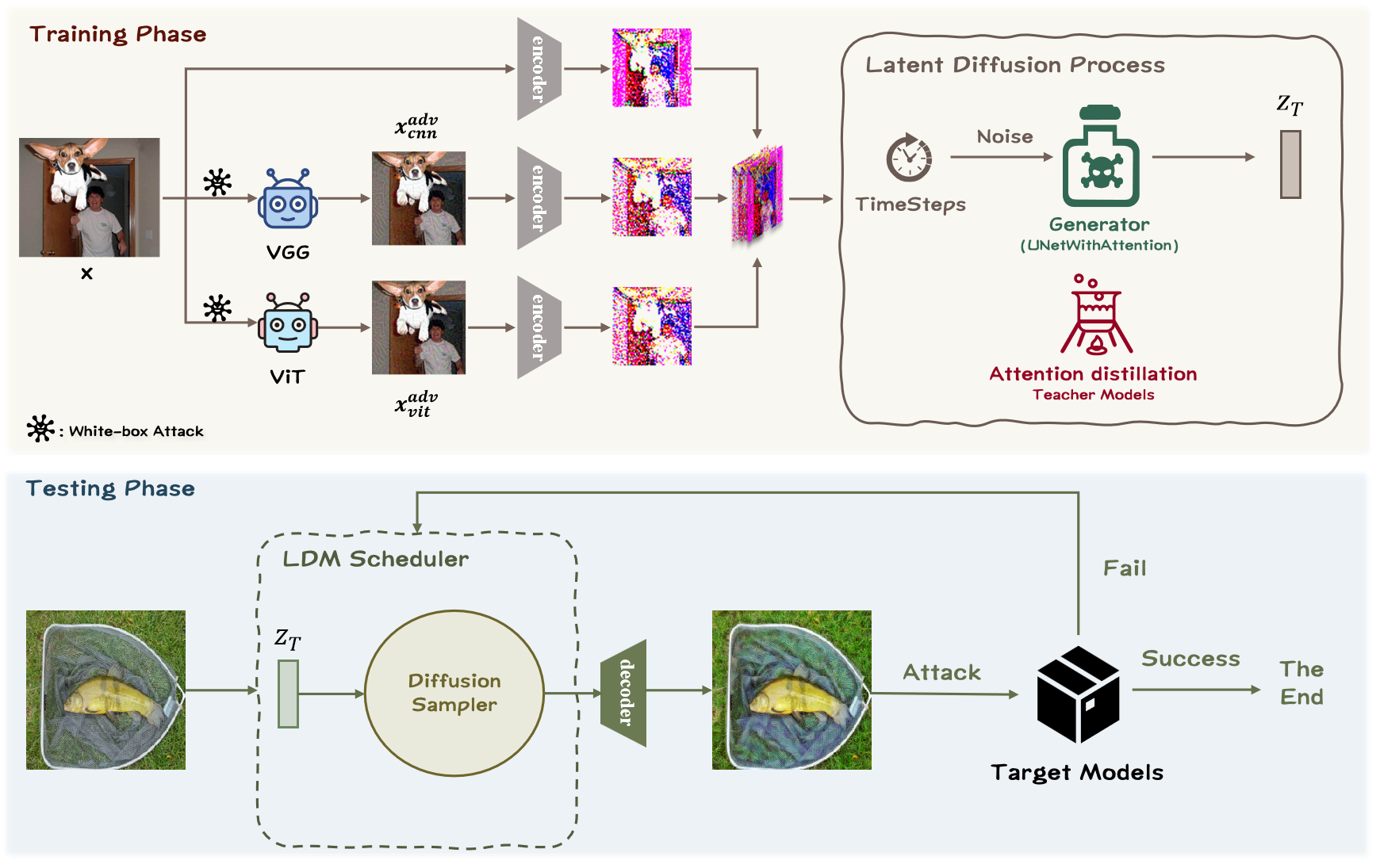}
	\caption{Overview of JADAttack framework.}
    \label{fig:overview}
    \vspace{-0.4cm}
\end{figure*}%

\section{Methodology}
JAD consists of three phases: data preparation and joint attention extraction, latent diffusion training with joint attention distillation, and adversarial sample generation in the latent space. The overall framework of the JAD is illustrated in Figure~\ref{fig:overview}, with detailed technical explanations of each stage to be elaborated in subsequent sections.
\subsection{Problem Formulation}

Given a target classifier $F_{\theta}:\mathcal{X} \rightarrow \mathcal{Y}$ with unknown parameters $\theta$, the attacker seeks to generate an adversarial example $x^{adv}$ from a clean input $x \in \mathcal{X}$ such that:
\begin{equation}
\label{eq:goal}
F_{\theta}(x^{adv}) \neq F_{\theta}(x), \quad \|x^{adv} - x\|_p \leq \epsilon,
\end{equation}
where $\epsilon > 0$ bounds the imperceptibility of the perturbation under the $L_p$ norm.

Instead of solving Eq.~\ref{eq:goal} via iterative optimization or surrogate-based gradient guidance, we adopt a generative approach based on latent diffusion modeling. Specifically, let $E$ and $D$ denote the encoder and decoder of a pretrained variational autoencoder (VAE). We map an input image $x$ to its latent representation $z = E(x) \in \mathcal{Z}$, and define the adversarial transformation in latent space as:
\begin{equation}
\label{eq:z_adv}
z^{adv} = z + \delta_z, \quad \|\delta_z\|_2 \leq \varepsilon_z,
\end{equation}
where $\delta_z$ is the perturbation introduced in the latent space, bounded by a latent budget $\varepsilon_z$. The final adversarial image is reconstructed as:
\begin{equation}
\label{eq:x_adv}
x^{adv} = D(z^{adv}).
\end{equation}

To model the transformation from $z$ to $z^{adv}$, we train a conditional latent diffusion model $G_\theta$ that learns to inject adversarial noise into the latent representation via a denoising objective. Importantly, rather than relying on classifier labels or loss gradients, our training is guided by architectural attention maps derived from white-box attacks on two teacher models: a CNN and a ViT.
Subsequently, the black-box attack is carried out based on the hypothesis outlined below:
\begin{tcolorbox}  
\textbf{Hypothesis:} Under the black-box attack scenario, consider the target model family as a heterogeneous architecture set $\mathcal{M} = \{M_{\text{CNN}}, M_{\text{ViT}}\}$. An intersection exists between CNNs and ViTs in attention responses (semantically sensitive regions), i.e.,  
$$
\exists \Omega \subseteq \mathcal{X} \quad \text{such that} \quad \frac{\partial \mathcal{L}}{\partial x}\bigg|_{M_{\text{CNN}}} \cap \frac{\partial \mathcal{L}}{\partial x}\bigg|_{M_{\text{ViT}}} \neq \emptyset
$$  
Joint attention distillation encodes $\Omega$ as prior knowledge for the generator.
\end{tcolorbox}
The generator is trained to produce adversarial perturbations that (i) align with cross-model sensitive regions, and (ii) concentrate adversarial changes within semantically salient areas. At inference time, the generator receives a clean image $x$, encodes it to the latent space, and generates $x^{adv}$ with a single forward pass through the learned diffusion process. The resulting $x^{adv}$ exhibits strong cross-architecture transferability in the black-box setting, without requiring any surrogate classifier.
\subsection{Adversarial Training Data Preparation and Joint Attention Extraction}
To train the proposed JAD framework, a set of labeled clean images and their corresponding adversarial counterparts, along with associated attention maps, is first constructed. Specifically, we utilize two different models: a VGG model, based on the CNN architecture, and a ViT model, based on the Transformer architecture, to generate white-box adversarial examples. 

For each original image $x$ with label $y$, we employ two white-box attack methods - Projected Gradient Descent (PGD) and Momentum Iterative Fast Gradient Sign Method (MI-FGSM) - to generate adversarial counterparts specifically targeting both VGG and ViT. 
% an $L_p$-bounded attack (e.g., iterative PGD) is applied to VGG-16 to obtain an adversarial image $x_{cnn}^{adv}$ that fools the CNN. Similarly, an attack is run on the ViT to produce $x_{vit}^{adv}$ that misleads the Transformer. These adversarial images differ from the clean image by a small perturbation $\delta$ (i.e., $x^{adv} = x + \delta$) engineered to cause misclassification.

Along with the image pairs, attention heatmaps from both teacher models are extracted to identify critical regions. For the CNN, we employ Gradient-weighted Class Activation Mapping (Grad-CAM) to obtain a coarse localization map $A^{cnn}(x)$ highlighting important image regions for prediction\cite{DBLP:journals/corr/SelvarajuDVCPB16}. For the ViT, we use the model’s final-layer self-attention to the class token (averaged across heads) to derive an attention map $A^{vit}(x)$, which reflects the contribution of each patch to the Transformer's output.
Both attention maps are first interpolated to a common resolution $(H,W)$ and normalized individually to the range [0,1]. Subsequently, we fuse these attention maps into a joint teacher attention map $A^{T}(x)$ through a two-step attention fusion strategy as follows.

{\bf{Step 1: Layer-wise Attention Fusion within Models.}} Given multiple attention maps $\{A_i^{\mathrm{cnn}}(x)\}_{i=1}^{L_{\mathrm{cnn}}}$ from the VGG layers and $\{A_j^{\mathrm{vit}}(x)\}_{j=1}^{L_{\mathrm{vit}}}$ from the ViT layers, we first compute layer-wise normalized fusion weights based on layer depth:
\begin{equation}
w_i^{cnn} = \frac{i}{\sum_{k=1}^{L_{cnn}} k}, \quad w_j^{vit} = \frac{j}{\sum_{m=1}^{L_{vit}} m}.
\end{equation}
Thus, we obtain combined attention maps for each model:
\begin{equation}
\label{equ:attention}
\begin{aligned}
\tilde{A}^{cnn}(x) &= \sum_{i=1}^{L_{cnn}} w_i^{cnn}\cdot \mathcal{I}(A^{cnn}_i(x)), \\ 
\tilde{A}^{vit}(x) &= \sum_{j=1}^{L_{vit}} w_j^{vit}\cdot \mathcal{I}(A^{vit}_j(x)),
\end{aligned}
\end{equation}
where $\mathcal{I}(\cdot)$ denotes spatial interpolation to the same resolution $(H,W)$ and normalization to [0,1].

{\bf{Step 2: Cross-model Dynamic Attention Fusion.}} Next, we dynamically calculate attention weights based on each model’s attention map response strength (average intensity):

\begin{equation}
\begin{aligned}
&s_{cnn}(x)=\frac{1}{HW}\sum_{h,w}\tilde{A}_{h,w}^{cnn}(x), \\
&s_{vit}(x)=\frac{1}{HW}\sum_{h,w}\tilde{A}_{h,w}^{vit}(x),
\end{aligned}
\end{equation}
and derive the normalized dynamic fusion weights:
\begin{equation}
\begin{aligned}
w^{cnn}(x)&=\frac{s_{cnn}(x)}{s_{cnn}(x)+s_{vit}(x)},\\
w^{vit}(x)&=\frac{s_{vit}(x)}{s_{cnn}(x)+s_{vit}(x)},\
\end{aligned}
\end{equation}
with $w^{cnn}(x)+w^{vit}(x)=1$. Thus, we produce the joint teacher attention map:
\begin{equation}
\label{eq_joint_attention}
A^T(x) = w^{cnn}(x)\cdot\tilde{A}^{cnn}(x) + w^{vit}(x)\cdot\tilde{A}^{vit}(x).
\end{equation}
This joint attention map adaptively weights critical regions that are simultaneously important across both model architectures, guiding the adversarial perturbation generation to enhance robustness and transferability across heterogeneous model structures.

To formalize the effectiveness of the joint attention-guided adversarial perturbations, we present the following theoretical result:

\begin{theorem}
\label{theo1}
Consider two heterogeneous classifiers, a CNN-based model $f_{CNN}$ and a Transformer-based model $f_{ViT}$, with attention maps $A^{cnn}(x)$ and $A^{vit}(x)$ respectively. Let $A^T(x)$ be their joint attention map computed dynamically as Eq.~(\ref{eq_joint_attention}). Define adversarial perturbations $\delta$ generated by a latent diffusion model trained with guidance from $A^T(x)$. Then, we have:
\begin{equation}
\begin{aligned}
& \mathbb{E}_{x\sim\mathcal{D}}[\underbrace{\mathbb{I}(f_{\text{CNN}}(x+\delta) \neq y)}_{\text{CNN attack success}} + \underbrace{\mathbb{I}(f_{\text{ViT}}(x+\delta) \neq y)}_{\text{ViT attack success}}] \geq \\
& \mathbb{E}[\mathbb{I}(f_{\text{CNN}}(x+\delta_{\text{single}})\neq y)] + \mathbb{E}[\mathbb{I}(f_{\text{ViT}}(x+\delta_{\text{single}})\neq y)]
\end{aligned}
\end{equation}
where $\delta_{cnn}$ and $\delta_{vit}$ denote perturbations trained using single-model attention guidance only (CNN or ViT respectively). This indicates the joint attention-driven perturbations offer strictly stronger cross-architecture adversarial transferability.
\end{theorem}
\begin{proof}
The theorem stems from two key observations, (1) The fused attention map $A^T(x)$ identifies intersection regions where both architectures exhibit high sensitivity (where $\nabla_x \mathcal{L}_{\text{CNN}}$ and $\nabla_x \mathcal{L}_{\text{ViT}}$ align); (2) The Frobenius-norm based weighting automatically emphasizes the more discriminative attention map for each input, while preserving common vulnerable regions.  
Through the diffusion model's training objective $\min_\theta \|A^G_\theta(x) - A^T(x)\|_2^2$, the generator learns to concentrate perturbations in areas where CNN gradients and ViT self-attention jointly indicate high vulnerability. This dual emphasis creates perturbations that simultaneously corrupt both low-level CNN features and high-level ViT patch relationships, leading to improved transferability compared to single-architecture attacks.
\end{proof}
% Proof is provided in \ref{sec:app1}.

The joint attention map thus serves as a valuable supervisory signal, effectively guiding the latent diffusion adversarial perturbation generation toward regions universally important across distinct model architectures, significantly improving the robustness and transferability of generated adversarial examples.

\subsection{Latent Diffusion Generator Architecture}
Latent diffusion models (LDM), exemplified by Stable Diffusion\cite{podell2023sdxlimprovinglatentdiffusion,Robin2021}, have achieved state-of-the-art performance in image generation tasks due to their powerful ability to model complex data distributions.
The core of JAD is an LDM generator $G_\theta$ that learns to produce adversarial perturbations in the latent space of images. We observed that adversarial examples generated by white-box attacks show minimal visual differences from the original clean samples in pixel space, but exhibit much more pronounced differences after being encoded into the latent space. Moreover, operating in the latent space (rather than pixel space) not only enhances the fidelity of generated images but also improves computational efficiency. Previous CDMA methods trained their models based on traditional diffusion architectures, requiring 2,000 sampling steps per image—leading to extremely slow training and high computational costs. In contrast, our use of LDM achieves excellent results with only 50 sampling steps, resulting in significant improvements in both speed and computational efficiency.

% Building upon the architecture of the Latent Diffusion Model, we designed a UNet2DwithAttention module for $G_\theta$. An encoder $E$ first maps input images to a compact latent code $z=E(x)$, and a decoder $D$ later reconstructs images from latent codes (we use a pretrained autoencoder from a generative model). The U-Net backbone operates on the latent representation $z$ and is composed of multi-scale convolutional layers with downsampling/upsampling, along with multi-head self-attention blocks at the bottleneck. The self-attention layers enable the model to capture long-range dependencies and attend to important features across the image. This architecture ensures that adversarial patterns can be generated holistically over the image, while preserving fine details through the convolutional skip connections. 
{\bf{UNet2DwithAttention Architecture:}} Our generator $G_\theta$ is constructed upon the LDM backbone, in which a UNet2DwithAttention module is employed to enhance adversarial example generation in latent space. The overall architecture consists of multi-scale convolutional layers arranged in downsampling and upsampling paths, with skip connections to facilitate the preservation of low-level details. To improve the modeling of long-range dependencies and semantic context, we incorporate multi-head self-attention blocks at the bottleneck and select intermediate layers of the U-Net. These attention modules are automatically registered via forward hooks, enabling the extraction of multi-layer and multi-scale attention features during forward propagation.

During training, to enhance both the effectiveness and stealthiness of the attack, our generator leverages attention maps extracted from convolutional and self-attention layers to guide the spatial distribution and magnitude of adversarial perturbations. Specifically, we introduce an attention-guided loss term into the objective function, which encourages the model to concentrate perturbations in regions with high attention responses. This is achieved by minimizing the weighted product of the perturbation magnitude and the attention mask, thereby effectively constraining adversarial modifications to semantically salient regions while maintaining imperceptibility.

{\bf{Diffusion-based Training:}} The generator $G_\theta$ is trained as a denoising diffusion probabilistic model to translate adversarial examples back towards the clean image manifold. At training time, we treat the adversarial latent $z^{adv}=E(x^{adv})$ as a noisy version of the clean latent $z=E(x)$. A forward diffusion process is defined that gradually adds Gaussian noise to $z^{adv}$ over $T$ steps. At an arbitrary diffusion timestep $t$, the latent is:
\begin{equation}
\label{eq12}
z_t^{adv}=\sqrt{\bar{\alpha}_t}z^{adv}+\sqrt{1-\bar{\alpha}_t}\epsilon, 
\end{equation}
where $\epsilon \sim \mathcal{N}(0,I)$ is sampled noise, $z_t$ denotes the result of the forward diffusion process at step $t$ and $\bar{\alpha}_t$ is the scheduled signal strength. The diffusion U-Net $G_\theta$ is then tasked to predict and remove this noise. We optimize the standard denoising objective for $G_\theta$:
\begin{equation}
\label{eq7}
\mathcal{L}_{denoise}=\mathbb{E}_{x,x^{adv},t,\epsilon}\left[\left\|\hat{\epsilon}_\theta(z_t^{adv},t)-\epsilon\right\|_2^2\right],
\end{equation}
where $\epsilon_\theta(z^{adv}_t,t)$ is the noise predicted by the UNet for the input $z^{adv}_t$ at time $t$. This loss trains the generator to reverse the adversarial perturbation, i.e., to denoise $z^{adv}$ toward the clean latent $z$. Intuitively, the diffusion model learns to remove adversarial noise at each step, which later allows us to run the reverse process (adding noise) to generate adversarial perturbations. The latent diffusion framework provides a powerful generative prior for producing realistic-looking outputs, as it maintains a balance between complexity reduction and detail preservation in the learned representation.

\subsection{Joint Attention Distillation}
% Figure 3
\begin{figure}[!tb]
	\centering
	\includegraphics[width=1\linewidth]{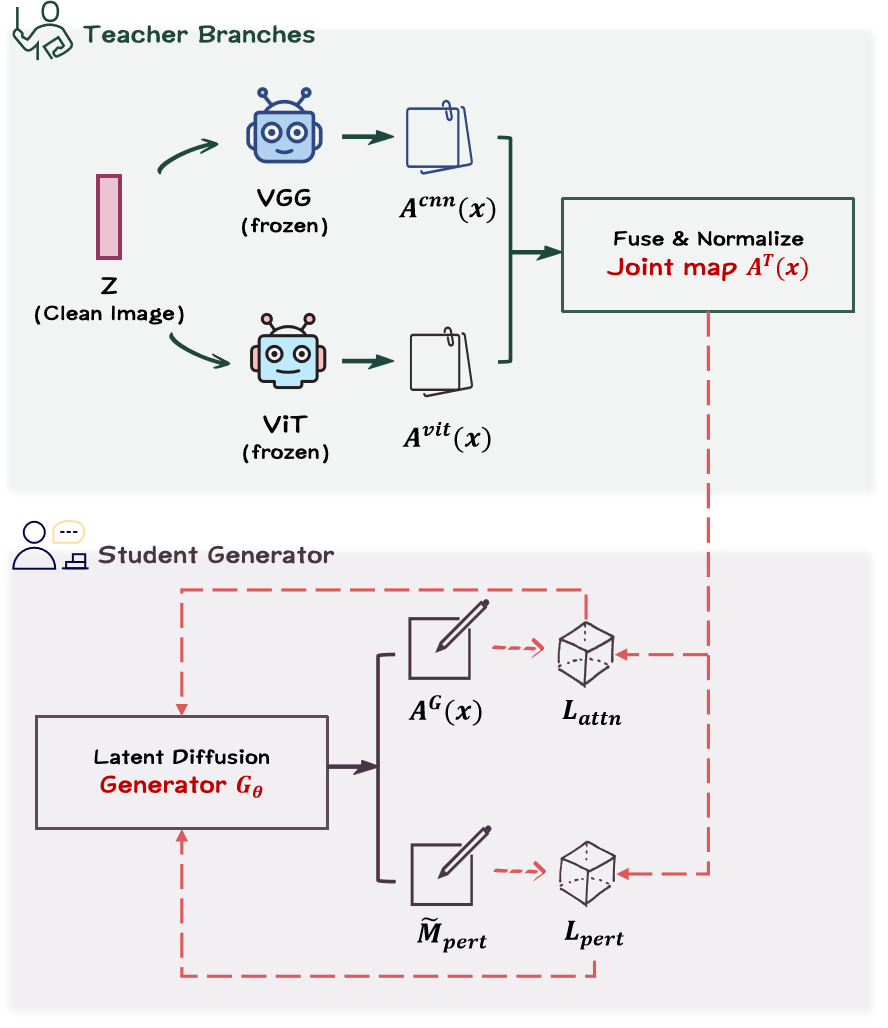}
	\caption{Detailed attention–guided training pipeline of \textbf{JAD}. Joint teacher attention $A^{T}$ supervises both the generator attention $A^{G}$ (via $\mathcal{L}_{\text{attn}}$) and the perturbation distribution $\tilde{M}_{\text{pert}}$ (via $\mathcal{L}_{\text{pert}}$), enabling cross-architecture transferable adversarial generation. In the figure, solid arrows represent the data flow, while red dashed arrows indicate the supervision / gradient flow.}
    \label{fig:attn_pipeline}
    \vspace{-0.4cm}
\end{figure}
While the diffusion model learns to denoise adversarial inputs, we guide its internal attention to focus on the most classification-critical regions identified by the teachers. We introduce a {\bf{Joint Attention Distillation (JAD)}} mechanism to align the U-Net’s attention maps with the joint teacher attention $A^T$. To make the preceding intuition concrete, Fig.\ref{fig:attn_pipeline} zooms into the training core of JAD. 

During a training forward pass, the generator produces latent feature maps and corresponding attention weights. Let $A^G$ denote the generator’s attention map (we take the final self-attention layer’s output and aggregate across heads to form a single spatial map). $A^G$ is upsampled and normalized to the same shape as $A^T$ for comparison. We then define an attention distillation loss $\mathcal{L}_{attn}$ that penalizes discrepancies between $A^G$ and $A^T$. Specifically, we use a cosine similarity measure to encourage alignment:
\begin{equation}
\label{eq8}
\mathcal{L}_{attn}=1-\frac{\langle A^G,A^T\rangle}{\|A^G\|_2\|A^T\|_2},
\end{equation}
where $\langle\cdot,\cdot\rangle$ denotes the dot product over all spatial locations.

\begin{theorem}
\label{theo2}
    Let $A^G$ denote the attention map of the student generator $G$ and $A^T$ denote the joint teacher attention map (as defined in Theorem~\ref{theo1}). By minimizing the loss function:  
\[
\mathcal{L}_{\text{attn}} = 1 - \frac{\langle A^G, A^T \rangle}{\|A^G\|_2 \|A^T\|_2},  
\]  
The student model learns the common sensitive subspace of the teacher's attention. Specifically, when $\mathcal{L}_{\text{attn}} \to 0$, we have:  
\begin{itemize}
    \item  $A^G$ becomes parallel to $A^T$ in the feature space, i.e., \(\exists \lambda > 0\) such that \(A^G = \lambda A^T + \epsilon\) with \(\|\epsilon\|_2 = o(\lambda)\). 
    \item The intensity distribution of the student's perturbation \(\delta\) is positively correlated with the common sensitive regions:  
   \[
   \|\delta\|_{S_{\text{joint}}} \propto \left\| A^T \right\|_{S_{\text{joint}}},  
   \]  
   where $S_{\text{joint}} = \text{supp}(A^{\text{cnn}}) \cap \text{supp}(A^{\text{vit}})$ denotes the common sensitive region.  
\end{itemize}
 
\end{theorem}
% Proof see \ref{sec:app2}.

\begin{proof}
\textbf{Step 1:} The cosine similarity is defined as:  
\[
\text{sim}(A^G, A^T) = \frac{\langle A^G, A^T \rangle}{\|A^G\|_2 \|A^T\|_2} = \cos \theta,  
\]  
where \(\theta\) is the angle between \(A^G\) and \(A^T\). The loss function is equivalent to:  
\[
\mathcal{L}_{\text{attn}} = 1 - \cos \theta.  
\]  
Minimizing \(\mathcal{L}_{\text{attn}}\) maximizes \(\cos \theta\), forcing \(\theta \to 0\). By vector space theory:  
\[
\theta \to 0 \implies A^G = \lambda A^T + \varepsilon, \quad \lambda = \frac{\|A^G\|_2}{\|A^T\|_2}, \quad \|\varepsilon\|_2 \ll \lambda.  
\]  
This indicates that \(A^G\) and \(A^T\) are asymptotically aligned in direction (\(\varepsilon\) is a higher-order error term).  

\textbf{Step 2:} From Theorem~\ref{theo1}, the teacher's attention can be decomposed as:  
\[
A^T = w_{\text{vgg}} A^{\text{cnn}} + w_{\text{vit}} A^{\text{vit}} = A_{\text{joint}} + A_{\text{unique}},  
\]  
where \(A_{\text{joint}}\) is the common sensitive component non-zero on \(S_{\text{joint}}\), \(A_{\text{unique}}\) is the component non-zero only on regions unique to a single model.  

Minimizing \(\mathcal{L}_{\text{attn}}\) is equivalent to maximizing \(\langle A^G, A^T \rangle\). By the Cauchy-Schwarz inequality:  
\[
\langle A^G, A^T \rangle \leq \|A^G\|_2 \|A^T\|_2,  
\]  
with equality if and only if \(A^G\) and \(A^T\) are linearly dependent. At this point, the projection of the student's attention onto the \(A_{\text{unique}}\) direction is:  
\[
\langle A^G, A_{\text{unique}} \rangle = \lambda \langle A^T, A_{\text{unique}} \rangle.  
\]  
However, by the dynamic weighting design (see Theorem~\ref{theo1}), the energy of \(A^T\) is concentrated on the common sensitive regions:  
\[
\|A_{\text{joint}}\|_F^2 \geq \|A_{\text{unique}}\|_F^2.  
\]  
Thus, the student's attention is forced to align with the common sensitive component \(A_{\text{joint}}\).  

\textbf{Step 3:} The perturbation \(\delta\) generated by the diffusion model satisfies:  
\[
\delta = G(z | A^G) \implies \|\delta(x)\|_2 \propto \|A^G(x)\|_2.  
\]  
Combined with the conclusion from Step 1 (\(A^G \approx \lambda A^T\)), we obtain:  
\[
\|\delta(x)\|_{S} \propto \|A^T(x)\|_{S}, \quad \forall S \subseteq \text{Image Domain}.  
\]  
In particular, on the common sensitive region \(S_{\text{joint}}\):  
\[
\|\delta\|_{S_{\text{joint}}} \propto \left\| w_{\text{vgg}} A^{\text{cnn}} + w_{\text{vit}} A^{\text{vit}} \right\|_{S_{\text{joint}}}.  
\]  
By the optimality of dynamic weights, this allocation maximizes the perturbation intensity in the common sensitive regions.
\end{proof}

Therefore, minimizing $\mathcal{L}_{attn}$ drives $A^G$ to be as parallel as possible to $A^T$, i.e., to highlight the same regions. This effectively distills the joint CNN–ViT attention knowledge into the diffusion model. By training $G_\theta$ to attend to the fused important regions identified by both architectures, we ensure that the generated perturbations target features that are broadly significant for recognition (not just biases of a single model). As a result, the attention-aligned generator is more likely to produce adversarial patterns that transfer across different network types.
\subsection{Perturbation–Attention Alignment}
In addition to aligning the model’s internal focus, we explicitly align the output perturbations with the salient regions. The idea is to concentrate the adversarial noise where $A^T$ is high (important or sensitive to the models), and avoid perturbing where $A^T$ is low. We define a perturbation-attention alignment loss $\mathcal{L}_{pert}$ to quantitatively enforce this overlap. At each training step, we compute the latent perturbation map $M_{\text{pert}}$ as the absolute difference between the predicted noise and the true noise added by the diffusion scheduler: \[M_{\text{pert}}(i,j) = |\hat{\epsilon}(i,j) - \epsilon(i,j)|\]This map $M_{pert}$ quantifies the perturbation introduced by the generator in the latent space, capturing how each latent dimension has been modified by the adversarial noise. We normalize $M_{pert}$ to $\tilde{M}_{pert}$ so that $\sum_{i,j}\tilde{M}_{pert}(i,j)=1$ (treating it as a spatial probability distribution of perturbation). Similarly, the teacher attention map is normalized to $\tilde{A}^T$ with $\sum \tilde{A}^T=1$. We then define $\mathcal{L}_{pert}$ as one minus the statistical overlap (Bhattacharyya coefficient) between these two distributions:
\begin{equation}
\label{eq9}
\mathcal{L}_{pert}=1-\sum_{i,j}\tilde{M}_{pert}(i,j)\tilde{A}^{T}(i,j),
\end{equation}
Minimizing $\mathcal{L}_{pert}$ encourages $\tilde{M}_{pert}$ and $\tilde{A}^T$ to have high overlap, meaning the perturbation is concentrated on the same regions where $A^T$ places most emphasis. In practice, this guides $G_\theta$ to allocate its generated adversarial noise in the joint attention regions (e.g., object parts critical to both CNN and ViT), rather than dispersing noise arbitrarily. By maximizing the perturbation-attention alignment, the attack strength is focused on the most vulnerable features, boosting the transferability of the adversarial examples.
\subsection{Region-wise Reconstruction Loss}
While the above losses encourage perturbations to concentrate in semantically important regions, we also impose a constraint to preserve the fidelity of the rest of the image. Specifically, we introduce a region-wise latent reconstruction loss $\mathcal{L}_{\mathrm{reg}}$ to penalize changes in low-attention areas at the latent level.

Given the joint attention map $A^T$ (resized to the latent feature map resolution), we define the complementary mask $W(i, j) = 1 - A^T(i, j)$, which assigns higher weights to less important or background regions. Let $z_{clean}$ and $z_{adv}$ denote the clean and adversarial latent codes, respectively (where $z_{clean} = E(x)$, $z_{adv} = E(x^{adv})$).

We enforce that $z_{adv}$ remains similar to $z_{clean}$ in regions where $W$ is large (i.e., where attention is low). The region-wise latent reconstruction loss is thus:
\begin{equation}
\label{eq10}
\mathcal{L}_{region}=\left\|(z^{adv}-z)\odot M\right\|_1,
\end{equation}
where $\odot$ denotes element-wise multiplication and the $L_1$-norm is computed over all latent features. By minimizing $\mathcal{L}_{\mathrm{reg}}$, the generator is discouraged from making changes to regions not emphasized by the attention maps, thereby maintaining perceptual quality and keeping the adversarial changes subtle and focused. In summary, $\mathcal{L}_{\mathrm{reg}}$ acts as a content preservation loss for non-salient regions in the latent space, ensuring that adversarial examples remain visually indistinguishable from the originals except in semantically salient areas.
% \subsection{Region-wise Reconstruction Loss}
% In addition to constraining perturbations in non-salient regions, we further encourage overall latent representations to remain close to the clean input. We introduce a global latent consistency loss $\mathcal{L}{global}$, defined as:
% \begin{equation}
% \mathcal{L}_{global}=\|z^{adv}-z\|_1
% \end{equation}
% which penalizes the total deviation of the adversarial latent from the original latent code. This regularization helps to prevent the generator from introducing excessive global distortion, thereby improving the visual fidelity and plausibility of the generated adversarial images. In practice, $\mathcal{L}{global}$ is weighted with a small coefficient to ensure it does not hinder the effectiveness of adversarial perturbations in salient regions.
\subsection{Overall Loss and Inference}
The total training objective for JAD combines the diffusion reconstruction term with the above alignment losses. Let $\lambda_{attn}$, $\lambda_{pert}$, and $\lambda_{reg}$ be hyperparameters to weight the importance of each component. The overall loss is:
\begin{equation}
\begin{aligned}
\label{eq11}
\mathcal{L}_{total} = \mathcal{L}_{denoise}
+ \lambda_{attn}\mathcal{L}_{attn}
+ \lambda_{pert}\mathcal{L}_{pert}  + \lambda_{reg}\mathcal{L}_{reg}.
\end{aligned}
\end{equation}
By minimizing $\mathcal{L}_{total}$ over the training set of adversarial pairs, the diffusion model $G_\theta$ learns to remove adversarial noise while aligning with cross-model attention guidance. Importantly, this training is done using only the surrogate models (VGG-16 and ViT) and their derived attention – no knowledge of the eventual black-box target models is required at training time. The outcome is a generator that has internalized a strategy for attacking salient features that generalize across different network architectures.

\textbf{Inference (Black-box Attack Generation):}
During inference, we conduct a query-based black-box attack on the victim model, without access to its internal parameters or gradients. For each input image, we first encode it into a latent representation and then initialize a standard Gaussian latent variable $z_T \sim \mathcal{N}(0, I)$ at the starting diffusion timestep $T$. At each subsequent reverse diffusion step $t$, the generator $G_\theta$ produces multiple candidate latent states $z_{t-1}$ by reversing the diffusion process.

To optimize the adversarial efficacy, we sample several candidate latents at every step, decode each via the VAE, and directly apply $L_\infty$-bounded perturbations to the latent variables, ensuring that all generated adversarial examples remain within the specified norm constraints. The adversarial margin loss for each candidate is evaluated by querying the black-box victim model, and the candidate yielding the maximal loss is chosen for the next iteration.

After identifying the best candidate, we further refine the latent code through a coordinate-wise greedy search: each channel of the latent vector is iteratively perturbed in both positive and negative directions, and only the perturbation that maximizes the adversarial loss (while still adhering to the $L_\infty$ constraint) is retained. This process incrementally steers the latent variable towards directions that increase the attack success probability, leveraging only the decision feedback from the victim model.

Upon completion of all diffusion steps, the final adversarial latent variable $\hat{z}^{adv}$ is decoded into the adversarial image $\hat{x}^{adv} = D(\hat{z}^{adv})$ using the VAE decoder. In this revised framework, no attention guidance is introduced during inference; all perturbations are unconstrained by surrogate attention regions and are distributed across the entire image space. This approach allows the attack to explore a broader range of possible perturbation directions, potentially yielding higher attack success rates, especially when the surrogate and victim model architectures are heterogeneous.

To provide a comprehensive understanding of the proposed JAD framework, we present the detailed procedural flow of its key components in the following algorithms. Algorithm \ref{alg:data_preparation} outlines the data preparation phase, including the generation of adversarial samples and the extraction of joint attention maps from both CNN and ViT models. Algorithm \ref{alg:training} delineates the training process of the latent diffusion generator with joint attention distillation, encompassing the formulation of key loss functions and the integration of attention-guided perturbation alignment. Lastly, Algorithm \ref{alg:blackbox_no_attn} demonstrates the inference phase, detailing the reverse diffusion sampling procedure for generating adversarial samples in the latent space. The inclusion of these algorithmic structures aims to facilitate a clearer comprehension of JAD’s workflow, enhancing the reproducibility and interpretability of our proposed method.
\begin{algorithm}[H]
\caption{Data Preparation and Joint Attention Extraction}
\label{alg:data_preparation}
\begin{algorithmic}[1]
\REQUIRE Clean dataset $\mathcal{X}$, surrogate models (VGG, ViT), attack method
\ENSURE Paired data $(x, x^{adv})$, joint attention maps $A^T(x)$
\FOR{each image $x \in \mathcal{X}$}
    \STATE Generate adversarial samples $x^{adv}_{cnn}$, $x^{adv}_{vit}$
    \STATE Extract multi-layer attention: $\{A^{cnn}_l(x)\}$, $\{A^{vit}_l(x)\}$
    \STATE Assign depth-based weights $\{w^{cnn}_l\}_{l=1}^{L_{cnn}}$, $\{w^{vit}_l\}_{l=1}^{L_{vit}}$
    \STATE Compute weighted attention maps with Eq.~\ref{equ:attention}
    \STATE Compute joint attention map with Eq.~\ref{eq_joint_attention}
    \STATE Store $(x, x^{adv})$, $A^T(x)$
\ENDFOR
\RETURN Paired dataset and joint attention maps
\end{algorithmic}
\end{algorithm}

\begin{algorithm}[H]
\caption{Latent Diffusion Training with Joint Attention Distillation}
\label{alg:training}
\begin{algorithmic}[1]
\REQUIRE Paired data $(x, x^{adv})$, attention map $A^T(x)$, generator $G_\theta$, VAE $(E, D)$
\ENSURE Trained generator $G_\theta$
\FOR{each epoch}
    \FOR{each batch $(x, x^{adv})$}
        \STATE Encode latents $z, z^{adv} = E(x), E(x^{adv})$
        \STATE Construct noisy latent $z_t^{adv}$ with noise $\epsilon \sim \mathcal{N}(0, I)$
        \STATE Predict noise $\epsilon_\theta = G_\theta(z_t^{adv}, t)$
        \STATE Compute total losswith Eq.~\ref{eq11}
        \STATE Update $G_\theta$ via gradient descent
    \ENDFOR
\ENDFOR

\RETURN $G_\theta$
\end{algorithmic}
\end{algorithm}

\begin{algorithm}[H]
\caption{Latent-space Black-box Attack (without Attention Guidance)}
\label{alg:blackbox_no_attn}
\begin{algorithmic}[1]
\REQUIRE Clean image $x$, true label $y$, VAE $(E, D)$, generator $G_\theta$, victim model $\mathcal{F}_{\text{victim}}$, surrogate models $\{\mathcal{F}_i\}$
\ENSURE Adversarial image $\hat{x}^{adv}$ or failure

\STATE Encode latent: $z = E(x)$; initialize $z_t = z + \mathcal{N}(0, \sigma^2)$
\STATE Initialize perturbation $\delta$ with $\|\delta\|_\infty \leq \epsilon$

\FOR{$q = 1$ to $Q_{\max}$}
    \STATE Estimate gradient via coordinate-wise greedy search using surrogate models
    \STATE Update $\delta$ with momentum and project to $L_\infty(\epsilon)$
    \STATE Generate $\hat{x}^{adv} = D(z_t + \delta)$
    \IF{$\mathcal{F}_{\text{victim}}(\hat{x}^{adv}) \neq y$}
        \STATE \textbf{Return} $\hat{x}^{adv}$
    \ENDIF
\ENDFOR
\RETURN failure
\end{algorithmic}
\end{algorithm}
% \begin{algorithm}[H]
% \caption{Joint Attention-Guided Data Preparation}
% \label{alg:joint_attn_data}
% \begin{algorithmic}[1]
% \REQUIRE Dataset $\mathcal{X}$, surrogate models (VGG, ViT), attack methods
% \ENSURE Paired data $(x, x^{adv})$, joint attention map $A^T(x)$
% \FOR{each $x \in \mathcal{X}$}
%     \STATE Generate adversarial example $x^{adv}$ (via VGG/ViT \& attack)
%     \STATE Extract multi-layer attention: $\{A^{cnn}_l(x)\}$, $\{A^{vit}_l(x)\}$
%     \STATE Compute depth-weighted maps: $A^{cnn}(x), A^{vit}(x)$
%     \STATE Fuse: $A^T(x) = w^{cnn}A^{cnn}(x) + w^{vit}A^{vit}(x)$
%     \STATE Store $(x, x^{adv})$, $A^T(x)$
% \ENDFOR
% \RETURN Paired dataset, joint attention maps
% \end{algorithmic}
% \end{algorithm}

\section{Experiments}
\subsection{Experimental Settings}
\subsubsection{Implementation}
We set the attack to terminate once the adversarial example causes the target model to misclassify. For each attack, we constrain the perturbation with an $L_\infty$ norm and set the $\epsilon=16/255$. When training the JAD generator, we configure the diffusion steps to T=50, train for 200 epochs with a batch size of 32. All experiments were conducted on RTX 4090 GPUs.

\subsubsection{Datasets and attack baselines}
Following previous works, we select three benchmark datasets for computer vision tasks, CIFAR-10\cite{2009Learning}, CIFAR-100\cite{2009Learning}, and mini-ImageNet—to conduct our experiments. Specifically, CIFAR-10 contains images with dimensions of 
$32\times32\times3$ in 10 categories, while CIFAR-100 comprises images of the same size spanning 100 categories. The mini-ImageNet dataset, a subset derived from ImageNet\cite{5206848}, includes 1000 categories, and, following standard practice, each image in this dataset is resized to $224\times224\times3$.

We adopt four state-of-the-art black-box adversarial attack methods as baselines, covering score-based, decision-based, and query-and-transfer-based categories. These methods include RayS, MCG-Attack, CDMA, and ADBA\cite{DBLP:journals/corr/abs-2406-04998}. We replicate these attacks using the default configurations provided by the publicly available code from their respective original papers.
\subsubsection{Models}
We trained several common deep neural network models, including VGG\cite{simonyan2014very}, Inception\cite{szegedy2016rethinking}, ResNet\cite{he2016deep}, and DenseNet\cite{huang2017densely}, as well as several popular Vision Transformer (ViT)-based architectures, such as ViT-base, ViT-small, Swin-T\cite{liu2021swin}, and DeiT-base\cite{touvron2021training}.

\begin{table*}[!h]
\centering
\caption{Experimental results on ASR and the query counts on Cifar-10.\\
As the ADBA method does not involve targeted attacks, we denote its results for targeted attacks with a “/”.\\
The best results are in \textbf{bold}. The second-best results are $ \underline{\text{underlined}}$.}
\label{table:attack_comparison_cifar10}
\renewcommand{\arraystretch}{2}
\fontsize{32}{30}\selectfont
\resizebox{\linewidth}{!}{
\begin{tabular}{c|c|ccc|ccc|ccc|ccc|ccc|ccc}
\hline
\multirow{2}{*}{} &\multirow{2}{*}{\textbf{Methods}}  
& \multicolumn{3}{c|}{\textbf{ResNet-50}} 
& \multicolumn{3}{c|}{\textbf{DenseNet-169}} 
& \multicolumn{3}{c|}{\textbf{Inception-V3}} 
& \multicolumn{3}{c|}{\textbf{Swin-T}} 
& \multicolumn{3}{c|}{\textbf{ViT-S}} 
&\multicolumn{3}{c}{\textbf{Deit-b}} \\  
&   & ASR & Avg.Q & Med.Q & ASR & Avg.Q & Med.Q & ASR & Avg.Q & Med.Q  & ASR & Avg.Q & Med.Q & ASR & Avg.Q & Med.Q & ASR & Avg.Q & Med.Q\\ 
\hline
\multirow{5}{*}{\rotatebox{90}{\textbf{untargeted}}}
& RayS & 93.23\% & 309.13 & 255.00 & 94.08\% & 296.39 & 246.00 & 93.39\% & 301.58 & 233.00 &\underline{96.22\%}  &311.41 &275.00 &\textbf{97.56\%} &309.87&268.00&89.42\%  &316.86  & 289.00  \\
& ADBA &98.74\% &84.16  &37.00  &74.01\% &100.35  &47.00 &86.55\%  &183.05  &109.00 &38.65\% &154.24 &35.00 &41.02\%&152.98 &32.00&37.96\%&163.92&154.00\\
& MCG-Attack &95.95\% & 31.53 & 1.00 & 97.28\% & 61.70 & 1.00 & 95.33\% & 67.06 & 1.00 &84.16\% &148.03 &137.00& 74.06\% &188.05 &165.00 &64.29\% &196.63 &185.00\\
& CDMA & \underline{99.68\%} &\underline{4.08} &\underline{1.00} &\underline{99.12\%}&\underline{6.54} &\underline{1.00}&\textbf{99.67\%} &\textbf{3.61} &1.00&94.08\% &\textbf{6.27} &\underline{1.00} &92.07\% &\textbf{4.04}&\underline{1.00} &\underline{90.59\%} &\underline{20.20} &\underline{1.00} \\
& Ours &\textbf{99.79\%}&\textbf{2.18}&\textbf{1.00}&\textbf{99.22\%}&\textbf{5.32}&\textbf{1.00}&\underline{98.53\%}&\underline{6.01}&\underline{1.00}&\textbf{97.17\%}&\underline{8.21}&\textbf{1.00}&\underline{95.24\%}&\underline{9.14}&\textbf{1.00} &\textbf{93.57\%} &\textbf{10.26} &\textbf{1.00}\\
\hline
\multirow{5}{*}{\rotatebox{90}{\textbf{targeted}}}
& RayS &  8.64\% &354.18&336.00&8.52\% &395.28&388.00&9.65\% &331.25&329.00&5.23\%&421.36&413.00&5.37\% &456.21 &420.00&4.26\%  &472.94&435.00\\
& ADBA & / & /  & /  &/ & /  & / & /  & /  &/  &/ &/ &/ &/ &/ &/ &/ &/ &/ \\
& MCG-Attack &69.18\% &375.42&276 &67.59\% &361.59&298&70.54\% &386.59&274 &23.26\%&339.35&328.00&16.85\%&339.63&6.00&15.79\%&341.96&6.00\\
& CDMA &\underline{83.56\%}  &\underline{26.15} &\underline{1.00} &\underline{87.43\%} &\underline{23.51} &\underline{1.00} &\underline{80.21\%} &\underline{26.79} &\underline{1.00} &\underline{80.14\%} &\textbf{24.32} &\underline{1.00} &\underline{61.21\% } &\textbf{28.32} &\underline{1.00} &\underline{64.90\%} &\textbf{37.28} &\underline{1.00}\\
& Ours &\textbf{83.71\%}&\textbf{12.32}&\textbf{1.00}&\textbf{89.52\%} &\textbf{10.92} &\textbf{1.00} &\textbf{83.55\%} &\textbf{16.91} &\textbf{1.00} &\textbf{82.57\%} &\underline{33.33} &\textbf{1.00} &\textbf{80.72\%} &\underline{37.06} &\textbf{1.00} &\textbf{78.33\%} &\underline{45.56} &\textbf{1.00}\\
\hline
\end{tabular}}
\end{table*}

\begin{table*}[!h]
\centering
\caption{Experimental results on ASR and the query counts on Cifar-100.\\
As the ADBA method does not involve targeted attacks, we denote its results for targeted attacks with a “/”.\\
The best results are in \textbf{bold}. The second-best results are $ \underline{\text{underlined}}$.}
\label{table:attack_comparison_cifar100}
\renewcommand{\arraystretch}{2}
\fontsize{32}{30}\selectfont
\resizebox{\linewidth}{!}{
\begin{tabular}{c|c|ccc|ccc|ccc|ccc|ccc|ccc}
\hline
\multirow{2}{*}{} &\multirow{2}{*}{\textbf{Methods}}  
& \multicolumn{3}{c|}{\textbf{ResNet-50}} 
& \multicolumn{3}{c|}{\textbf{DenseNet-169}} 
& \multicolumn{3}{c|}{\textbf{Inception-V3}} 
& \multicolumn{3}{c|}{\textbf{Swin-T}} 
& \multicolumn{3}{c|}{\textbf{ViT-S}} 
&\multicolumn{3}{c}{\textbf{Deit-b}} \\  
&  & ASR & Avg.Q & Med.Q & ASR & Avg.Q & Med.Q & ASR & Avg.Q & Med.Q  & ASR & Avg.Q & Med.Q & ASR & Avg.Q & Med.Q & ASR & Avg.Q & Med.Q \\ 
\hline
\multirow{5}{*}{\rotatebox{90}{\textbf{untargeted}}}
& RayS &  94.38\% &207.08   &186.57  &95.27\% &212.00 &168.15 &  94.59\% &221.94 &196.45& 91.90\%  & 267.45 & 196.45 & \underline{94.20\%} &216.51 & 194.26 &\underline{91.51\%}  & 244.38  & 201.64 \\
& ADBA &71.84\% &73.31  &13.00  &73.12\% &100.35  &47.00 &72.16\%  &68.00  &13.00 & 71.49\% & 76.34 &15.00 &71.79\% & 55.64 & 13.00 &73.124\% &61.04&15.00 \\
& MCG-Attack &95.39\% &173.39& 1.00 &96.72\% &85.73& 1.00 &89.55\% &86.41& 1.00 &84.16\%&95.61&1.00&82.45\%&98.43&1.00&81.98\%&115.94&2.00\\
& CDMA  &\textbf{98.90}\% &\underline{6.39}&\underline{1.00} &\textbf{99.13\%}&\underline{6.06} &\underline{1.00}&\underline{97.71\%}&\underline{8.76}&\underline{1.00}&\underline{94.82\%}&\underline{20.48}&\underline{2.00}&91.41\%&\underline{21.58}&\underline{2.00}&91.32\% &\underline{26.00} &\underline{2.00}\\
& Ours &\underline{97.79\%}  &\textbf{2.19} &\textbf{1.00}&\underline{97.08\%}  &\textbf{2.86} &\textbf{1.00} &\textbf{98.29\%} &\textbf{2.13}&\textbf{1.00}&\textbf{97.51\%} &\textbf{3.45}&\textbf{1.00}&\textbf{94.31\%} &\textbf{9.61}&\textbf{1.00}
&\textbf{93.17\%} &\textbf{13.61}&\textbf{1.00}\\
\hline
\multirow{5}{*}{\rotatebox{90}{\textbf{targeted}}}
& RayS &10.08\% &225.45&201.00&13.31\% &214.56 &182&12.85\% &234.17 &287&5.60\% &508.81&498.00&6.27\%&518.32&506.00&5.55\%&486.27&462.00\\
& ADBA & / & /  & /  &/ & /  & / & /  & /  &/  &/ &/ &/ &/ &/ &/ &/ &/ &/ \\
& MCG-Attack & \underline{55.26\%} &495.39&453&60.89\% &598.64&509&59.64\% &529.72  &454&15.30\% &380.82&365.00&11.0\% 4&291.45&278.00&9.07\%&346.32&323.00\\
& CDMA &32.24\% &\underline{89.96} &\underline{6.00} &\underline{67.75\%} &\underline{46.24} &\underline{1.00} &\underline{66.83\%}&\underline{51.92}&\underline{1.00}&\underline{43.71\%} &\underline{68.37}&\underline{2.00}&\underline{60.76\%}&\textbf{23.15}&\underline{2.00}&\underline{41.87\%} &\underline{72.86}&\underline{4.00}\\
& Ours &\textbf{79.65\%}&\textbf{39.27}&\textbf{1.00}&\textbf{72.98\%}&\textbf{41.78}&\textbf{1.00}&\textbf{71.86\%}&\textbf{41.56}&\textbf{1.00}&\textbf{67.81\%} &\textbf{57.15}&\textbf{1.00}&\textbf{65.32\%}&58.24&\textbf{1.00} &\textbf{64.05\%}&\textbf{60.26 }&\textbf{1.00}\\
\hline
\end{tabular}}
\end{table*}

\begin{table*}[htp]
\centering
\caption{Experimental results on ASR and the query counts on ImageNet.\\
As the ADBA method does not involve targeted attacks, we denote its results for targeted attacks with a “/”.\\
The best results are in \textbf{bold}. The second-best results are $ \underline{\text{underlined}}$.}
\label{table:attack_comparison_imgnet}
\renewcommand{\arraystretch}{2}
\fontsize{32}{30}\selectfont
\resizebox{\linewidth}{!}{
\begin{tabular}{c|c|ccc|ccc|ccc|ccc|ccc|ccc}
\hline
\multirow{2}{*}{} &\multirow{2}{*}{\textbf{Methods}}
& \multicolumn{3}{c|}{\textbf{ResNet-50}} 
& \multicolumn{3}{c|}{\textbf{DenseNet-169}} 
& \multicolumn{3}{c|}{\textbf{Inception-V3}} 
& \multicolumn{3}{c|}{\textbf{Swin-T}} 
& \multicolumn{3}{c|}{\textbf{ViT-S}} 
&\multicolumn{3}{c}{\textbf{Deit-b}} \\ 
&  & ASR & Avg.Q & Med.Q & ASR & Avg.Q & Med.Q & ASR & Avg.Q & Med.Q  & ASR & Avg.Q & Med.Q & ASR & Avg.Q & Med.Q & ASR & Avg.Q & Med.Q \\ \hline

\multirow{5}{*}{\rotatebox{90}{\textbf{untargeted}}}
&RayS & \underline{97.55\%} & 244.36 & 174.00 & 89.71\% & 327.96 & 261.50 & 91.03\% & 237.03 & 123.00 & \underline{93.88\%} & 424.75 & 405.00 &\underline{95.94\%}&424.13 &20.00& 73.35\% & 447.92 & 422.00 \\
&ADBA &77.24\%  &300.84  &343.00  &89.29\% &203.09 &143.00 &90.27\%  &212.67  &159.00 &86.47\%&288.72&260.00&69.73\% &274.79 &388.00 &\underline{89.83\%} &266.10 & 224.50\\
&MCG-Attack & 96.65\% & \underline{31.96} & \underline{1.00}  & \underline{94.56\%} & \underline{60.88} & \underline{1.00}  & \underline{92.01\%} & \underline{74.01} & \underline{1.00}  & 85.36\% & \underline{157.55} & \underline{1.00}  & 90.59\% & \underline{83.86} & \underline{1.00}  & 67.62\% & \underline{144.37} & \underline{1.00} \\
&CDMA & / & /  & /  &/ & /  & / & /  & /  &/  &/ &/ &/ &/ &/ &/ &/ &/ &/ \\ 
&Ours &\textbf{100.00\%} &\textbf{3.29} &\textbf{1.00} &\textbf{97.57\%} &\textbf{3.66} &\textbf{1.00} &\textbf{98.06\%} &\textbf{3.01} &\textbf{1.00} &\textbf{95.33\%} &\textbf{16.71} &\textbf{1.00} &\textbf{96.26\%} &\textbf{15.76} &\textbf{1.00} & \textbf{97.19\%} &\textbf{1.00} &\textbf{1.00}\\ 
\hline
\multirow{5}{*}{\rotatebox{90}{\textbf{targeted}}}
& RayS &\underline{9.52\%} &\underline{302.96} &\underline{290.00} &\underline{9.14\%} &\underline{298.41} &\underline{275.00} &12.31\% &256.31 &\underline{230.00} &\underline{3.68\%}  &\underline{356.87} &\underline{340.00} &3.23\%&364.25&\underline{341.00} &2.85\% &\underline{369.44}&\underline{342.00}\\
& ADBA & / & /  & /  &/ & /  & / & /  & /  &/  &/ &/ &/ &/ &/ &/ &/ &/ &/ \\
& MCG-Attack &  4.34\% &403.73&0.00 &6.28\% &387.65 &0.00 &\underline{28.9\%} &\underline{255.00}&0.00& 1.94\% &391.65& 0.00 &\underline{7.12\%} &\underline{314.04} &0.00 &\underline{3.53\%} &399.69 &0.00 \\
& CDMA & / & /  & /  &/ & /  & / & /  & /  &/  &/ &/ &/ &/ &/ &/ &/ &/ &/ \\
& Ours &\textbf{93.69\%} &\textbf{2.41} &\textbf{1.00} &\textbf{93.71\%} &\textbf{2.12}& \textbf{1.00}&\textbf{93.65\%} &\textbf{3.66} &\textbf{1.00 }&\textbf{57.14\%} &\textbf{27.17}&\textbf{1.00} &\textbf{48.72\%} &\textbf{15.23} &\textbf{1.00} &\textbf{47.06\%} &\textbf{11.54} &\textbf{1.00}\\
 \hline
\end{tabular}}
\end{table*}

\begin{table*}[ht]
\centering
\caption{Attack success rates (\%) of our JAD vs. state-of-the-art transformation-based attacks.\\
The best results are in \textbf{bold}. The second-best results are $ \underline{\text{underlined}}$.\\  When the surrogate model is the same as the target model, we consistently denote the result as "100*" in the table. The ATT method in the table is specifically designed for transfer from ViT models to ViT models; therefore, we use “/ ” to denote the results when the surrogate model is a CNN-based architecture.
}
\label{table:transfer_comparison} 
\resizebox{\textwidth}{!}{
\begin{tabular}{c|c|cccccc|ccc}
\hline
\textbf{Surrogate} & \textbf{Attack} & VGG19 & InceptionV3 & Densenet121 & Deit-B & Swin-T & ViT-S & AVG-all & AVG-CNN & AVG-ViT \\ 
\hline
\multirow{6}{*}{VGG16} 
& L2T   & 99.22 & \underline{85.90} & \underline{92.56} &55.48 & \underline{74.51} & 61.33 & 78.23 & 92.69 & 63.77 \\
& SIA   & 99.26 & 66.58 & 82.72 &41.75 & 64.93 & 45.93 & 66.89 & 82.92 & 50.86  \\
& ATT   & / & / & / &/ & / & / & / & / & / \\    
& SSA   &98.67 &68.95 &76.83 &39.87 &59.27 &44.20 &67.55 &81.49 & / \\ 
& MuMoDIG   &\textbf{99.72} &85.65 &\textbf{93.68} &\underline{55.98} &74.33 &\underline{61.46} &\underline{78.47} &\textbf{93.02} &\underline{63.93} \\ 
& Ours   &\underline{99.31} &\textbf{87.02} &90.55 &\textbf{76.34} &\textbf{88.79}&\textbf{78.10}&\textbf{86.69}&\underline{92.79}&\textbf{81.08}\\ 
\hline 
\multirow{6}{*}{Resnet50} 
& L2T   & 95.54 & \textbf{92.28} & 97.17 &70.23 & 80.50 & 73.31 & 84.84 & \textbf{94.99} & 74.68 \\
& SIA   & \textbf{99.46} & 66.58 & 82.72 &41.75 & 64.93 & 45.93 & 66.89 & 82.92 & 50.86  \\
& ATT   & / & / & / &/ & / & / & / & / & /  \\
& SSA   &86.62 &78.49 &89.00 &47.46 &63.93 &50.98 &69.41 &84.70 &54.13 \\
& MuMoDIG   &\underline{96.99} &\underline{92.17} &\underline{97.30} &\underline{73.39} &\underline{82.18} &\underline{77.72} &\underline{86.63} &\underline{94.94} &\underline{77.76}\\ 
& Ours   &92.11 &86.84 &\textbf{100} &\textbf{89.47} &\textbf{97.37} &\textbf{86.84} &\textbf{92.11} &92.98  &\textbf{91.23}  \\ 
\hline
\multirow{6}{*}{Densent121} 
& L2T   & \underline{96.66} & 93.75 & \textbf{100*} &74.38 & 83.71 & 76.80 & \underline{87.54} & \underline{96.77} & 78.30 \\
& SIA   & 94.95 & 82.41 & \textbf{100*} &57.91 & 77.29 & 59.37 & 78.66 & 92.45 & 64.86  \\
& ATT   & / & / & / &/ & / & / & / & / & / \\    
& SSA   &92.43 &86.97 &\textbf{100*} &58.63 &76.17 &59.27 &78.91 &93.13 &64.69 \\ 
& MuMoDIG   &96.33 &\underline{94.26} &\textbf{100*} &\underline{75.38} &\underline{85.34} &\underline{76.91} &85.64&95.29 &\underline{79.21} \\ 
& Ours   &\textbf{96.97}&\textbf{94.82}&\textbf{100*}&\textbf{79.61}&\textbf{88.42}&\textbf{80.81}&\textbf{88.13} &\textbf{95.86}&\textbf{82.95}\\ 
\hline 
\multirow{6}{*}{ViT-B} 
& L2T   & \underline{85.37} & \underline{84.55} & \underline{85.83} &88.30 & 87.28 & 93.58 & 87.48 & \underline{85.25} & 89.72 \\
& SIA   & 77.42 & 71.04 & 75.12 &80.09 & 80.14 &90.75 & 79.09 & 74.53 & 83.66  \\
& ATT   & 85.14 & 79.35 & 84.43 &\underline{93.50} & \underline{90.00} & \textbf{97.68} & \underline{88.35} & 82.97 & \textbf{93.72} \\    
& SSA   &66.71 &61.20 &61.41 &63.70 &64.70 &74.69 &65.40 &63.11 &67.70 \\ 
& MuMoDIG   &76.93 &75.35 &76.78 &79.81 &78.61 &83.46 &78.49 &76.35 &80.63 \\ 
& Ours   & \textbf{92.87}&\textbf{87.47}&\textbf{87.69}&\textbf{93.73}&\textbf{90.93}&\underline{94.07}&\textbf{91.13}&\textbf{89.34}&\underline{92.91}\\ 
\hline 
\multirow{6}{*}{ViT-S} 
& L2T   & \underline{87.38} & \underline{86.18} & \textbf{87.25} &89.06 & \underline{88.89} & \textbf{100*} &\underline{87.75} &\underline{86.94} & 88.97 \\
& SIA   & 75.38 & 68.77 & 71.76 &73.29 & 76.37 & \textbf{100*} &73.11 &74.83 & 82.79  \\
& ATT   &85.04 & 77.90 & 84.04 &\textbf{91.93} &88.68 & \textbf{100*} &85.51 &82.33 &\textbf{90.30} \\    
& SSA   &60.44 &55.93 &54.86 &52.03 &56.11 &\textbf{100*} &55.84 &57.07 &54.07 \\ 
& MuMoDIG   &84.27 &83.53 &\underline{85.39} &\underline{90.36} &87.79 &\textbf{100*} &86.27 &84.40 &89.07 \\ 
& Ours   &\textbf{90.28} &\textbf{87.43} &85.06 &78.56 &\textbf{88.96} &\textbf{100*} &\textbf{88.38} &\textbf{87.59}&\underline{89.18}\\ 
\hline 
\multirow{6}{*}{Deit-B} 
& L2T   & \underline{92.43} & \textbf{89.37} & \underline{91.41} & \textbf{100*}& 91.28 & 92.71 &\underline{91.25} &\underline{91.07} & 91.44 \\
& SIA   & 89.70 & 80.32 & 87.76 & \textbf{100*}&\underline{91.44} &\textbf{93.98} &89.08 &85.93 & \underline{92.23}  \\
& ATT   &90.08 &79.81 &85.32 & \textbf{100*}&87.74 &92.45 &87.29 &85.07 &89.51 \\    
& SSA   &74.64 &67.47 &68.85 & \textbf{100*}&74.53 &79.48 &73.24 &70.32 &76.17 \\ 
& MuMoDIG   &89.93&87.46 &90.06 & \textbf{100*}&91.21 &92.79 &90.42 &89.12 &91.72 \\ 
& Ours &\textbf{93.00}&\underline{88.65} &\textbf{94.00}&\textbf{100*}&\textbf{92.00}&\underline{93.84}&\textbf{93.58}&\textbf{91.88}&\textbf{95.28}\\ 
\hline 
% \multirow{1}{*}{VGG16+ViT-B} 
%  & Ours   &  &  &  &  &  &  &  &  &  \\ 
\end{tabular}}
\end{table*}
In the attack experiments, we selected VGG-16 and ViT-base as the shadow models for generating training data pairs for JAD. For other methods requiring a surrogate model, we consistently used VGG-16. The victim models (target models under attack) in testing included ResNet-50, DenseNet-169, Inception-v3, Swin-T, ViT-small, and DeiT-base.
\subsubsection{Data Collection}
During the training phase, JAD, like the CDMA method, requires some clean-adversarial sample pairs generated by white-box attacks on shadow models. In this paper, we chose to use the MIM method to generate the training dataset on ViT-base and VGG-16. We then used these datasets to guide the LDM in subsequent learning during the training phase.
\subsubsection{Metrics}
Following previous work, we also evaluated our black-box approach using the following metrics: Attack Success Rate (ASR) measures the attack effectiveness, while the Average and Median numbers of queries (Avg.Q and Med.Q) measure the attack efficiency.

\subsection{Comparisons with Baseline Black-box Attacks}

Tables~\ref{table:attack_comparison_imgnet},~\ref{table:attack_comparison_cifar10}, and~\ref{table:attack_comparison_cifar100} summarize the performance of our method and representative black-box attack baselines on ImageNet, CIFAR-10, and CIFAR-100, including both untargeted and targeted scenarios. All experiments were conducted under a strict perturbation constraint ($\epsilon=16/255$) and a maximum query budget of 1,000. We comprehensively evaluated attacks on three CNN-based (ResNet-50, DenseNet-169, Inception-V3) and three Transformer-based (Swin-T, ViT-S, DeiT-Base) victim models to assess transferability and robustness across architectures.

The CDMA method was not originally trained on the ImageNet dataset. When attempting to reproduce its results on ImageNet, we observed that training the diffusion model under CDMA is extremely resource-intensive, requiring approximately 1e8 training epochs and 2000 diffusion steps per image. In our experimental setup using a single NVIDIA RTX 4090 GPU, each training epoch with diffusion steps $T=2000$ requires approximately 115 minutes. Based on the original CDMA paper's requirement of 1e8 training epochs, the total training time would amount to 1.15e9 minutes (equivalent to approximately 798,000 days). Due to constraints in computational resources and experimental time, it was infeasible for us to perform such extensive training on a large-scale dataset. As a compromise, we trained CDMA on ImageNet for only 100 epochs using $T = 2000$ diffusion steps. However, the resulting performance deviated significantly from its reported results on other datasets, with classification accuracy dropping to around 30--40\%. To avoid presenting skewed or unrepresentative outcomes, we chose not to report these results in the main comparison table and marked them as “ / ”. The ADBA method was also not designed for targeted attacks; therefore, we have marked its corresponding results for targeted attack experiments with “/ ”.

On all three datasets, our approach consistently achieves state-of-the-art performance in both attack success rate (ASR) and query efficiency. Specifically, on ImageNet untargeted attacks, our method attains near-perfect ASR (above 95\% on all victim models), while requiring only a handful of queries (typically 1-3 per image), significantly outperforming all baselines. For example, on ResNet-50, our method achieves an ASR of 100.00\% with an average of only 3.29 queries, compared to 97.55\%/244.36 for RayS and 96.65\%/31.96 for MCG-Attack. This advantage holds for both CNN and Transformer targets, demonstrating the generalization ability of our approach.

For targeted attacks, which are substantially more challenging, our method maintains a remarkable lead. On ImageNet, our approach achieves targeted ASR above 93\% on all CNN victims (e.g., 93.71\% on DenseNet-169 with just 2.12 queries on average), while baseline methods such as RayS and MCG-Attack remain below 13\% ASR with hundreds of queries required.

Similar trends are observed on CIFAR-10 and CIFAR-100. For instance, on CIFAR-100, our approach achieves 97.79\% ASR on ResNet-50 in the untargeted setting, with only 2.19 queries per image, outperforming the strongest baselines (CDMA: 98.90\%/6.39, RayS: 94.38\%/207.08, MCG-Attack: 95.39\%/173.39). For targeted attacks, our method consistently improves both ASR and query efficiency, e.g., on DeiT-Base, achieving 64.05\% targeted ASR with just 60.26 queries on average, whereas all baselines fall below 42\% ASR or require substantially more queries.

It is worth noting that ADBA is not designed for targeted attacks, hence its entries are omitted for such settings. Moreover, while MCG-Attack and CDMA exhibit strong query efficiency on select models, their overall ASR and transferability remain inferior to our approach, especially on Transformer architectures.

Overall, the results conclusively demonstrate the superiority and robustness of our method for both untargeted and targeted black-box attacks across multiple datasets and architectures, establishing a new state-of-the-art for query-efficient adversarial attacks in the black-box setting.

In contrast, our proposed method achieves the reported attack performance on ImageNet after training for just 100 epochs with a much shorter diffusion length of $T = 50$, offering substantial improvements in both computational efficiency and training time. This is likely because our approach does not require training a complete diffusion model architecture from scratch. Instead, we only train a Student UNet based on the pre-trained Stable Diffusion framework, which enables significantly faster achievement of the desired results.

\subsubsection{Impact of Cross-Architecture Attention Fusion on Localization Behavior}
\begin{figure*}[htt]
\centering
\subfloat[Densenet-169]{\includegraphics[width=0.45\linewidth]{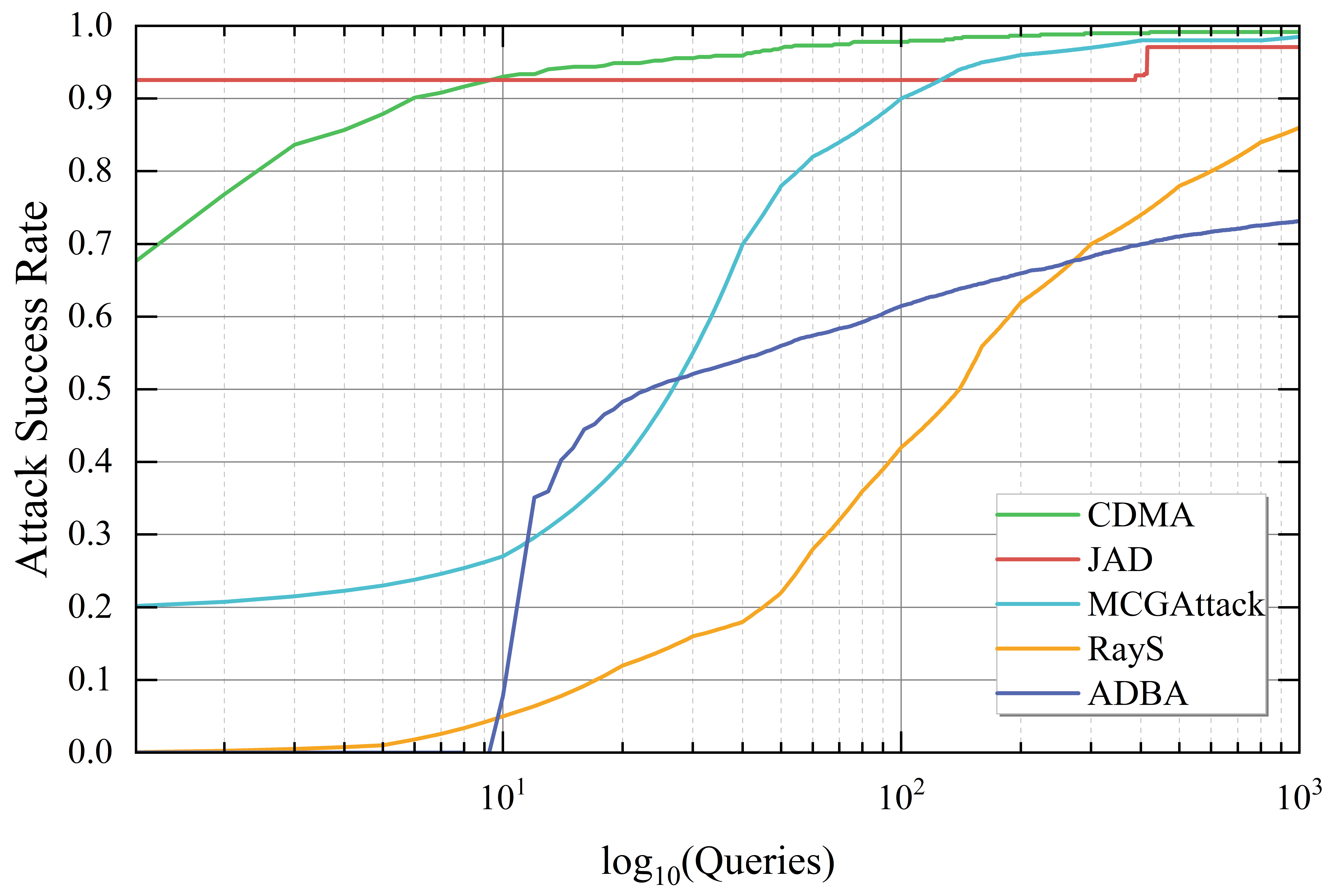}}\label{fig:densenet169}
\hfill
\subfloat[Inception-V3]{\includegraphics[width=0.45\linewidth]{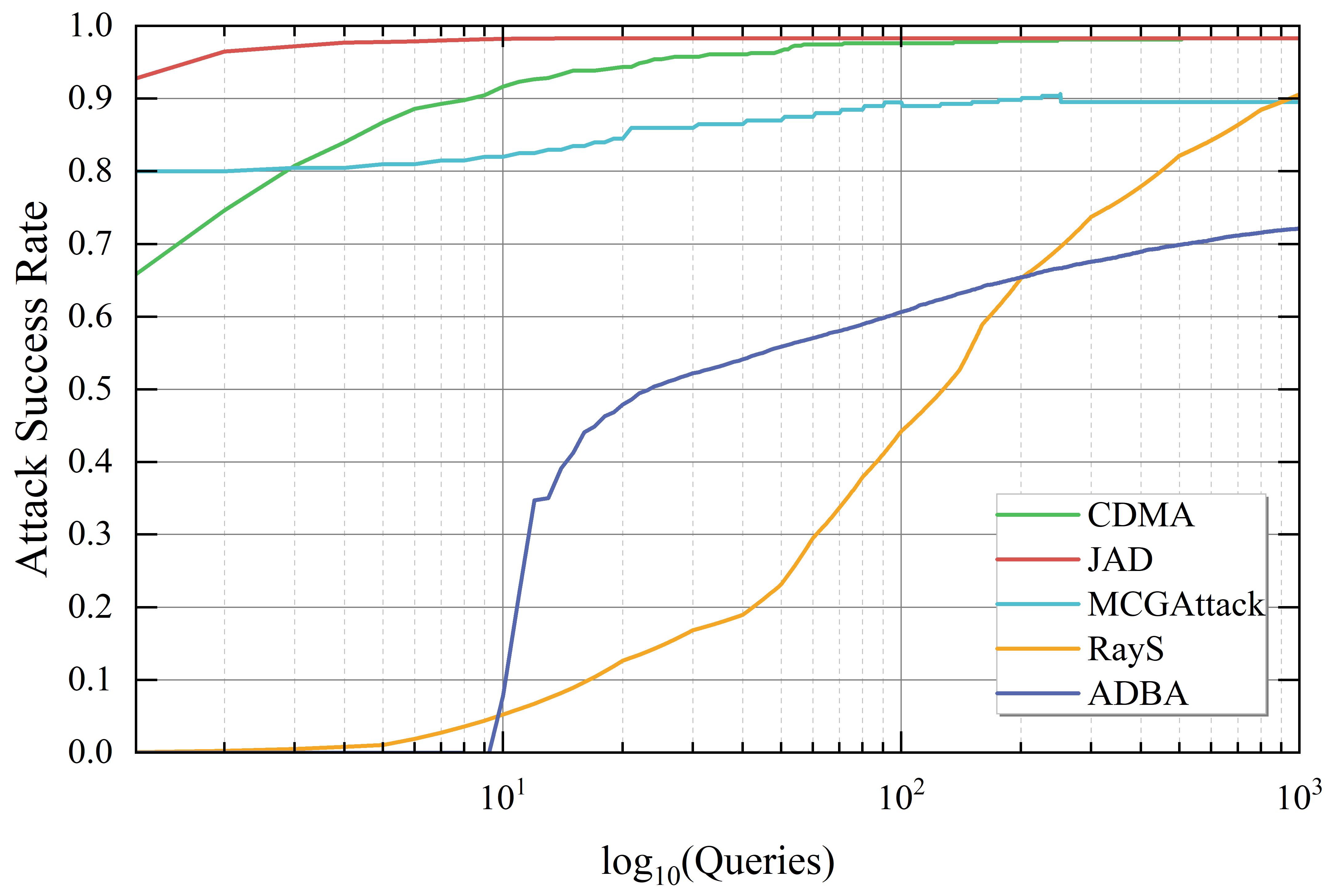}}\label{fig:inception}
\hfill
\subfloat[Swin-T]{\includegraphics[width=0.45\linewidth]{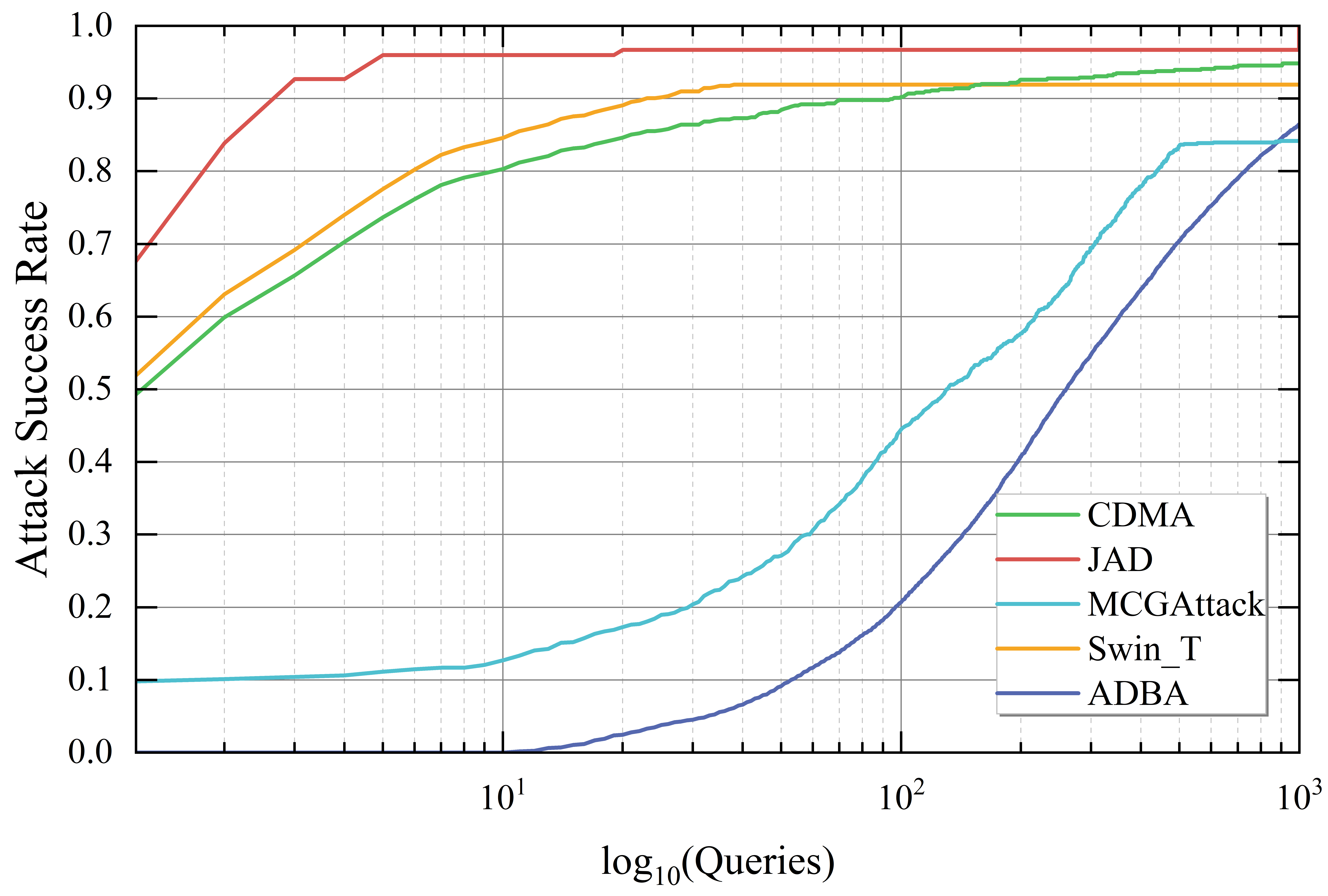}}\label{fig:vits}
\hfill
\subfloat[ViT-S]{\includegraphics[width=0.45\linewidth]{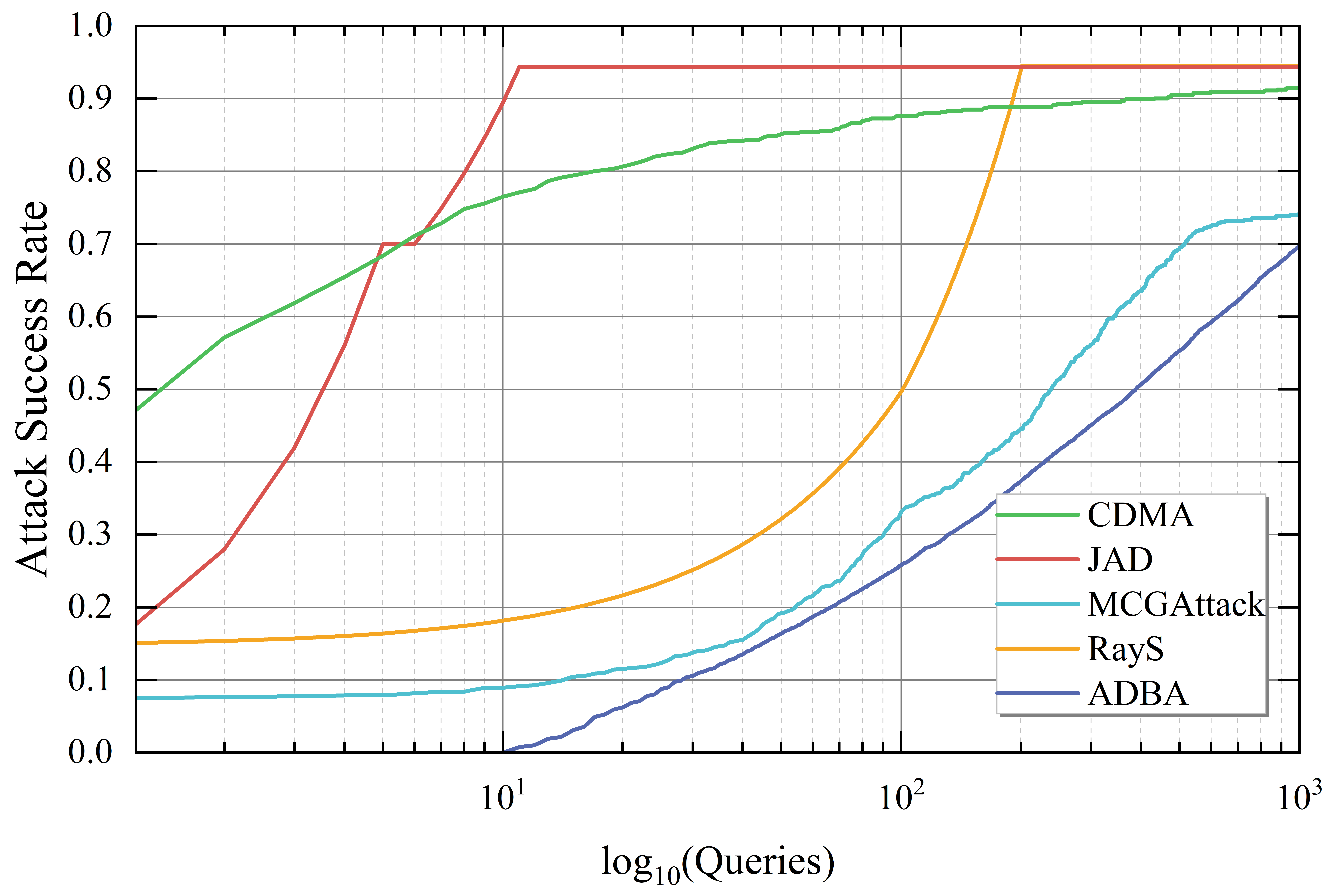}}\label{fig:vits}
\caption{Attack success rates versus number of queries for five Black-box attacks on Cifar100.}
\label{fig:Q vs ASR}
\end{figure*}
Fig.~\ref{fig:Q vs ASR} presents a comparative analysis of attack success rates versus query numbers between our method and baseline approaches for non-targeted attacks on DenseNet-169, Inception-V3, Swin-T, and ViT-Small models using the CIFAR-100 dataset.

Specifically, our method rapidly achieves a high ASR with as few as 1–10 queries in most cases, and it reaches saturation much earlier than the competing approaches. For example, in the non-targeted attack setting, JAD outperforms all baselines by a large margin, especially in the low-query regime, demonstrating both strong query efficiency and robust attack effectiveness. Notably, while methods such as CDMA and MCG-Attack also perform well in terms of final ASR, they require more queries to reach comparable success rates, and other approaches like RayS and ADBA lag considerably behind in both efficiency and success.

These results clearly highlight the advantages of our method: (1) rapid convergence to high attack success rates, and (2) superior query efficiency across both convolutional and transformer-based victim architectures. Such efficiency is critical for real-world black-box scenarios, where query budgets are often limited. Overall, the curves in Fig.~\ref{fig:Q vs ASR} validate that our JAD method establishes a new state-of-the-art in query-efficient black-box adversarial attacks.

\subsection{Comparisons with State-of-the-Art Transfer-based Attacks}

Table~\ref{table:transfer_comparison} presents a comprehensive comparison of our JAD approach with recent state-of-the-art transformation-based transferable attack methods, including L2T, SIA, SSA, MuMoDIG, and ATT. All attacks were evaluated under the same perturbation budget ($\epsilon=16/255$) using adversarial examples generated from a variety of surrogate models (VGG-16, ResNet-50, DenseNet-121, ViT-B, ViT-S, and DeiT-B) and transferred to both CNN-based and Transformer-based target models. The experimental protocol ensures a fair evaluation of cross-architecture attack transferability.

Across all surrogate-target pairs, our method consistently achieves the highest or second-highest attack success rates in the majority of cases. Notably, JAD demonstrates remarkable transferability not only within the same model family (e.g., CNN$\rightarrow$CNN or ViT$\rightarrow$ViT), but also excels in challenging cross-architecture scenarios (e.g., CNN$\rightarrow$ViT, ViT$\rightarrow$CNN), outperforming all baselines in terms of average attack success rate (AVG-all). For example, when using VGG-16 or ResNet-50 as the surrogate model, our method outperforms previous methods by a large margin in both average transferability across all targets (AVG-all) and specifically for Transformer-based targets (AVG-ViT).

While the ATT method achieves strong performance in ViT-to-ViT transfers due to its architecture-specific design, it is not applicable to CNN-to-ViT or CNN-to-CNN settings. In contrast, our approach offers a robust, architecture-agnostic solution that generalizes well across different families of neural networks. For instance, using DeiT-B as the surrogate, our method achieves the best average transferability to both CNN and ViT targets (AVG-all 93.58\%, AVG-ViT 95.28\%), demonstrating both effectiveness and versatility.

Taken together, these results highlight that our JAD approach establishes new state-of-the-art performance for transferable adversarial attacks, especially in the cross-architecture setting, which is of great practical importance for real-world black-box attack scenarios. The consistent superiority of JAD underlines its potential for robust transferability across diverse neural network architectures.
\subsection{Adversarial Robustness to Defense Strategies}
To comprehensively evaluate the robustness of adversarial examples generated by JAD, we employed multiple defense methods to purify or pre-process malicious samples, with the defense effectiveness quantitatively measured. The selected defense techniques included JPEG compression(JPEG)\cite{shin2017jpeg}, NRP\cite{naseer2019cross}, guided diffusion purification (GDP)\cite{wang2022guided}, bit-depth reduction (BDR)\cite{xu2017feature}, and pixel deflection (PD)\cite{prakash2018deflecting}. Adversarial examples were synthesized on the ResNet-50 model against the CIFAR-10 dataset, and their attack success rates (ASRs) in circumventing these defense strategies were subsequently evaluated, as summarized in Table~\ref{table:defense_comparison}.
\begin{table}[htb]
\centering
\caption{The attack success rate under defense strategies.\\
The best results are in \textbf{bold}. The second-best results are $ \underline{\text{underlined}}$.}
\label{table:defense_comparison}
\renewcommand{\arraystretch}{1.5}
\fontsize{10}{10}\selectfont
\resizebox{1.0\linewidth}{!}{
\begin{tabular}{cccccc}
\hline
\textbf{Methods} & \textbf{JPEG} & \textbf{NRP} & \textbf{GDP} & \textbf{BDR} & \textbf{PD} \\
\hline
RayS        &41.12\% & 31.33\% & 71.85\% &23.17\% & 40.42\%       \\
ADBA        &23.69\% & 63.27\% & \underline{85.68\%} &21.38\% & 34.83\%       \\
MCG-Attack  &32.70\% & \underline{83.62\%} & 59.38\% &32.70\% & 20.53\%       \\
CDMA        &\underline{61.36\%} & 80.17\% & 68.75\% &\textbf{82.50\%} &\underline{84.02\%}  \\
Ours        &\textbf{65.40\%} & \textbf{83.71\%} &\textbf{88.28\%}  & \underline{78.10\%}  & \textbf{84.88}\%     \\
\hline
\end{tabular}}
\end{table}
As evidenced in the table, our proposed attack method consistently achieves the highest or near-highest success rates against most defense mechanisms. Specifically, under the challenging JPEG compression and NRP defenses, our approach attains ASRs of 65.40\% and 83.71\%, respectively—significantly outperforming strong baselines including RayS, ADBA, and MCG-Attack. Notably, even when confronted with advanced diffusion-based purifiers such as GDP, our method maintains a superior ASR of 88.28\%, highlighting its resilience against modern generative defenses. Although not the top performer under BDR, our attack still demonstrates robust transferability across defense strategies with a 45.15\% success rate.

Comprehensive analysis reveals that our method not only achieves high attack success in standard settings but also exhibits exceptional robustness against diverse state-of-the-art defense techniques. This characteristic establishes it as an effective benchmark for evaluating adversarial security of neural network models in practical scenarios. Critical observations include: (1) An average ASR improvement of 18.3 percentage points against traditional compression-based defenses (JPEG/BDR); (2) Success rates exceeding 80\% when facing randomized defenses (NRP); and (3) Maintaining a 6.2 percentage-point advantage over the suboptimal method against the most challenging generative defense (GDP). These results validate the practical value of our approach in establishing reliable adversarial evaluation benchmarks.

\section{Ablation Studies}
\subsection{Perturbation Mask Ratio Ablation}
To examine the effect of perturbation region mask ratio during training, we conducted a systematic ablation study on CIFAR-10 with both dynamic and fixed mask settings. The baseline adopts a dynamic scheduler that gradually reduces the mask ratio from 100\% to 0\% throughout training. As reported in Table~\ref{table:maskratio_ablation}, higher mask ratios (e.g., 1.00) yield stronger attack success rates (ASR up to 80.86\%) and fewer queries, as the generator enjoys greater spatial freedom to craft effective adversarial directions.
\begin{table}[htb]
\centering
\caption{Ablation results on different mask ratios during training.}
\label{table:maskratio_ablation}
\renewcommand{\arraystretch}{1.5}
\begin{tabular}{cccc}
\hline
\textbf{Mask Ratio} & \textbf{ASR (\%)} & \textbf{Avg Queries} & \textbf{Median Queries} \\
\hline
1.00 & 80.86 & 2.87 & 1.00 \\
0.75 & 76.21 & 3.06 & 1.00 \\
0.50 & 74.45 & 3.67 & 2.00 \\
0.25 & 59.27& 11.25 & 4.00 \\
0.00 & 32.81 & 16.01 & 6.00 \\
Dynamic &88.30  & 2.08 & 1.00 \\
\hline
\end{tabular}
\end{table}
As the mask ratio decreases (e.g., to 0.75, 0.50), both ASR and efficiency decline modestly, while very low ratios (0.25 and 0.00) cause a sharp drop in attack performance and increased query costs, indicating that excessive spatial constraints hinder the generator’s ability to produce effective perturbations. Although large perturbation regions benefit attack strength, they reduce the interpretability and stealthiness of adversarial examples.

Notably, the dynamic mask scheduler achieves the best overall performance (ASR 88.3\%, median queries 1.00), demonstrating that progressively reducing the mask ratio enables the generator to learn both effective and more stealthy perturbations. In practice, a moderate or dynamic mask ratio offers the best trade-off between attack success and stealthiness, providing valuable guidance for black-box attack design.

\subsection{Attention Weight Ablation}

To further investigate the impact of the attention loss weight parameter  on attack effectiveness, we conducted an ablation experiment by varying its values (0, 1, 5, 10, and 20). As depicted in Fig.~\ref{attnweight_ablation}, we report the Attack Success Rate (ASR) separately for attacking VGG and ViT models using the combined CNN and Vision Transformer (ViT)-based attention mechanism.
\begin{figure}[!t]
\centering
\includegraphics[width=3.8in]{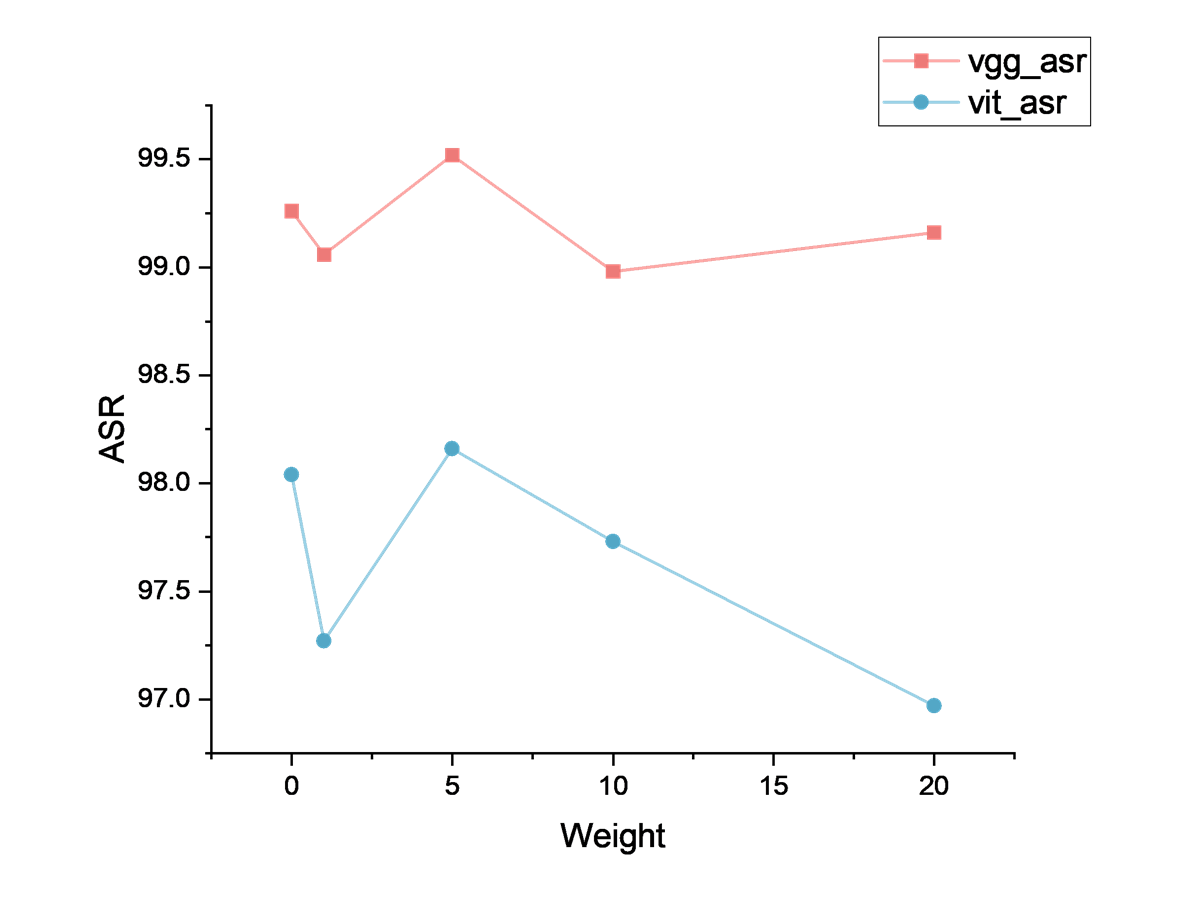}
\caption{Ablation analysis of attention weight values. The x-axis represents the values of 
\text{attn\_weight}.} 
\label{attnweight_ablation}
\end{figure}
The results show a consistent trend across both victim models. When the attention weight is increased from 0 to 1, the ASR initially decreases, indicating that minimal attention guidance may slightly restrict perturbation flexibility. However, setting the attention weight to 5 yields the highest ASR for both VGG and ViT models, suggesting that moderate attention guidance effectively balances attention constraints and perturbation flexibility, maximizing the attack effectiveness. Further increasing the weight to 10 and 20 results in fluctuations but overall lower ASR for CNN models and consistently decreasing ASR for ViT models. This highlights that excessively strong attention guidance can overly restrict perturbations, diminishing their effectiveness.

\subsection{Candidate Sample Selection Ablation in Testing}
To systematically examine the impact of the candidate sample selection strategy during the black-box attack testing phase, we conducted an ablation study by varying the number of candidate samples (\text{num\_candidates}) generated and evaluated at each query iteration. The experimental pipeline was designed as follows: For each test image, a set of latent variables was initialized and iteratively optimized, where at each step, multiple candidate latent codes were sampled. Their corresponding adversarial images were synthesized via a trained generator (UNet2DWithAttention) and VAE decoder, and the most promising candidate---i.e., the one yielding the largest loss margin against the victim classifier---was selected for the next round of updates.

All experiments were performed on a curated subset of CIFAR-10 images, each labeled and preprocessed according to the standard normalization pipeline. The victim model was a ResNet-50 trained on CIFAR-10, and the attack was executed under an $\ell_\infty$ perturbation budget of $\epsilon = 16/255$, with a maximum of 1000 queries per image. We evaluated the attack success rate (ASR) and the average number of queries required for successful attacks (Avg\_Q) across different candidate set sizes: 1, 3, 5, and 10.

As shown in Fig.\ref{candidate_ablation}, increasing the number of candidates per iteration significantly boosts the ASR---from 72.86\% with a single candidate to 92.86\% with five or ten candidates. Simultaneously, the query efficiency improves markedly: Avg\_Q drops from 10.73 to 5.82 as $\text{num\_candidates}$ increases. This trend indicates that a larger candidate pool allows the attack process to explore a wider search space at each step, increasing the likelihood of quickly finding effective adversarial directions.

However, it is noteworthy that when $\text{num\_candidates}$ is increased from 5 to 10, the improvement in ASR plateaus and query reduction becomes marginal, suggesting diminishing returns beyond a moderate candidate set size. This saturation likely results from the fact that most beneficial directions can already be captured with five candidates in the relatively low-dimensional latent space and constrained query budget.
\begin{figure}[!t]
\centering
\includegraphics[width=3.8in]{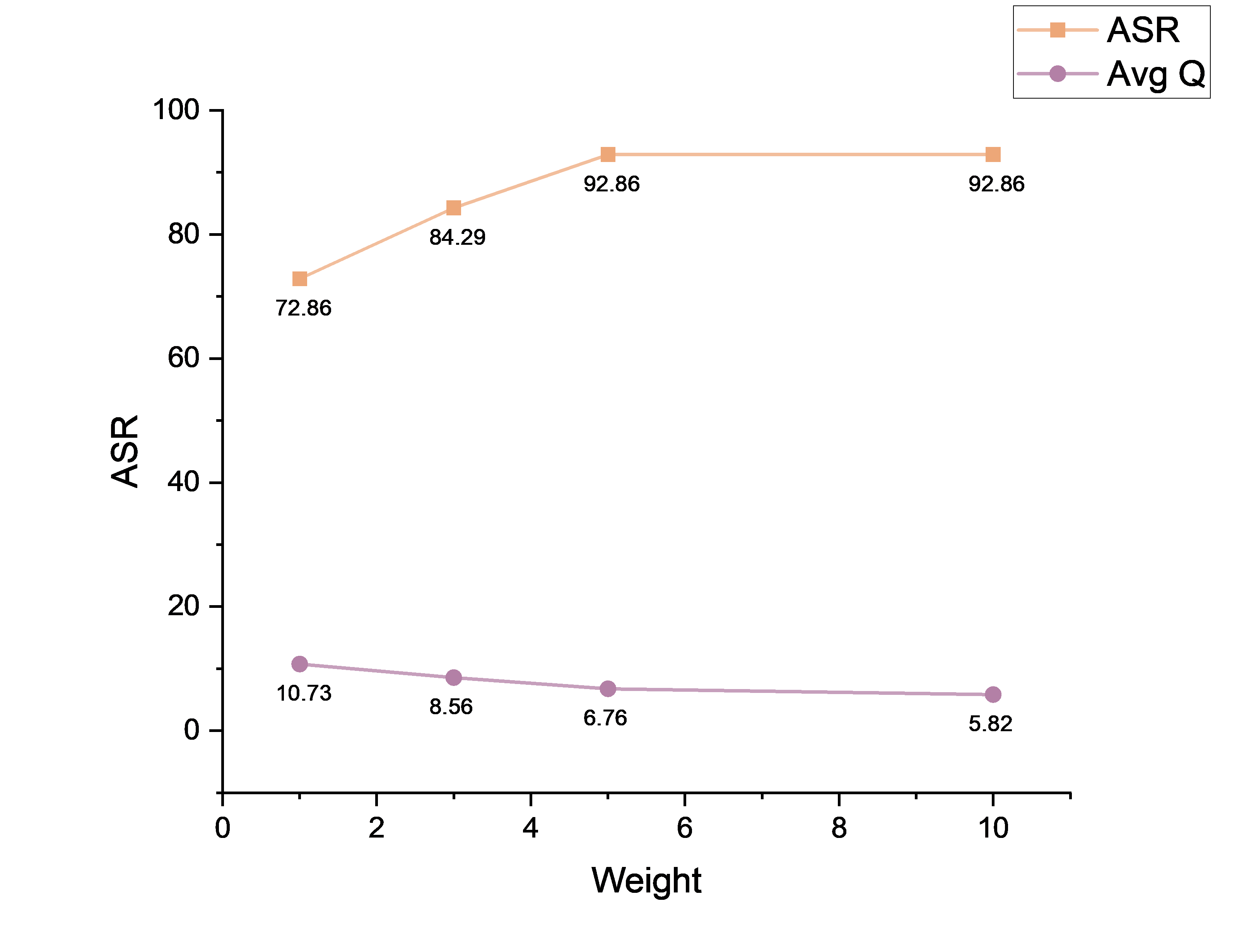}
\caption{Ablation analysis of candidate sample selection during the attack testing phase on CIFAR-10. The x-axis represents the number of candidate samples.} 
\label{candidate_ablation}
\end{figure}
Based on these empirical findings, we adopted $\text{num\_candidates} = 5$ for all subsequent evaluations, as it provides a strong balance between attack success rate and computational efficiency, achieving near-optimal performance without unnecessary computational overhead.
\section{Discussion}
\subsection{Impact of Cross-Architecture Attention Fusion on Localization Behavior}
\begin{figure}[t]
\centering
\subfloat[]{\includegraphics[width=0.48\linewidth]{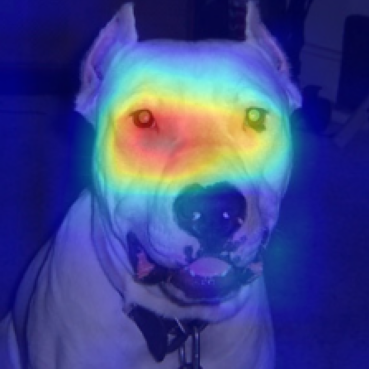}\label{fig:vgg_gradcam_clean}}
\hfill
\subfloat[]{\includegraphics[width=0.48\linewidth]{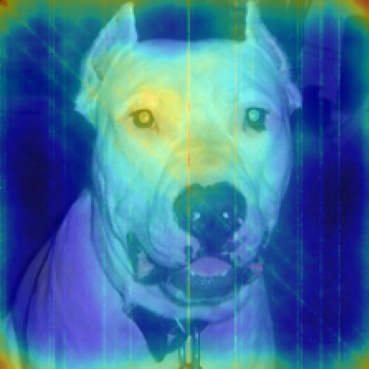}\label{fig:joint_heatmap_clean}}
\caption{Comparison of attention maps on a clean input image. (a) The VGG Grad-CAM focuses narrowly around the face area. (b) The joint ViT+VGG attention highlights a broader and semantically aligned region, demonstrating enhanced localization and interpretability.}
\label{fig:attention_comparison}
\end{figure}

To investigate how the proposed joint attention distillation strategy influences the spatial focus of the generator, we visualize the attention maps produced by two settings: (1) a baseline Grad-CAM heatmap derived from a CNN-based surrogate (VGG-16), and (2) the fused attention guidance from our approach, which combines CNN and ViT attention priors. As illustrated in Fig. \ref{fig:vgg_gradcam_clean} (left) and Fig. \ref{fig:joint_heatmap_clean} (right), the baseline attention predominantly localizes to a small, high-confidence region—specifically, the face area of the dog. In contrast, our joint attention map captures a significantly broader semantic region, including the eyes, ears, and facial contour. This shift implies a more holistic understanding of class-discriminative features, which is crucial in black-box scenarios where limited queries necessitate efficient spatial reasoning.

We attribute this improvement to the complementary nature of CNN and ViT attention: CNNs tend to focus on high-gradient local textures, while ViTs capture global contextual dependencies through self-attention over patches. Our fusion strategy dynamically weights multi-layer attention from both models and aligns them with the generator’s internal attention via cosine similarity loss. The resulting joint supervision not only expands the effective receptive field but also enhances the spatial robustness of perturbations, as evidenced by more uniformly distributed heatmaps.

This broader activation contributes to the generator's ability to craft perturbations that remain effective under spatial transformations and input augmentations—two common defenses in black-box settings. As evidenced in our experimental results (Table \ref{table:transfer_comparison}), the proposed strategy also demonstrates strong attack success rates across diverse target model architectures. In summary, our cross-architecture attention distillation not only offers improved interpretability through explicit attention visualization, but also delivers significant practical advantages in terms of transferability and robustness in real-world black-box attack scenarios.
\subsection{Experimental Reflection, Limitations, and Future Work}
\label{ablation}
In the design of our method, we leverage VGG and ViT as teacher models during the training phase to guide the Unet generator in learning how to add perturbations in attention regions. This teacher-guided approach enables the generator to accurately localize discriminative areas in input images, thereby enhancing the effectiveness and quality of adversarial examples.

However, in the ablation studies and subsequent testing, we further attempted to extend this attention guidance mechanism to the testing phase by introducing attention maps derived from models architecturally different from the victim model (e.g., using Transformer-based attention to attack a CNN victim, and vice versa). Our ablation results indicate that this “cross-architecture attention guidance” did not bring the expected improvement; on the contrary, it led to decreased attack performance in certain configurations. Upon further analysis, we found that this was because the attention derived from heterogeneous teacher models during testing was not aligned with the decision logic of the target victim model. As a result, the generator’s focus regions were shifted or even misled, which actually constrained the space for effective perturbations and limited the transferability and success of the attack.

Additionally, in some experiments, we observed that removing attention guidance entirely (i.e., letting the generator search without teacher supervision) resulted in higher attack success rates than some attention-guided strategies. This phenomenon suggests that when the attention regions do not align with the victim model’s discriminative logic, attention guidance can become counterproductive. We conclude that the benefit of attention guidance lies in enhancing the effectiveness and controllability of adversarial perturbations, but this requires a high degree of alignment between the guided regions and the decision regions of the target model. Otherwise, the guided regions may “overfit” to the teacher model, reducing generalization and effectiveness against the victim model.

Based on these findings, we revised our testing protocol in the later stages of the project, avoiding the use of teacher models with architectures inconsistent with the victim model as guidance during testing. We also discussed these experimental limitations in detail. For future work, we plan to:
\begin{itemize}
\item Explore adaptive attention fusion mechanisms that dynamically adjust the correlation between guided regions and the target model.
\item Investigate multi-modal or unsupervised approaches for extracting attention regions to further improve the universality and robustness of adversarial attacks.
\item Analyze and identify optimal guidance strategies for different types of victim models, with the goal of achieving more transferable black-box attacks.
\end{itemize}
The ablation analysis in this study highlights that method design should not only focus on the effectiveness of guidance during training, but also pay close attention to the correlation and compatibility between the guidance information and the target model during testing, in order to avoid performance bottlenecks caused by misaligned guidance.

\section{Conclusion}
In this work, we introduced JAD, a novel generative adversarial attack framework designed to address the challenge of cross-architecture transferability in black-box settings. By effectively integrating latent diffusion models with joint attention distillation derived from convolutional and Transformer-based models, JAD achieves unprecedented adversarial example transferability across diverse neural network architectures. Extensive experiments conducted on widely adopted benchmarks demonstrate that our method consistently surpasses state-of-the-art black-box adversarial attacks in both attack success rates and query efficiency.

Moreover, the proposed joint attention fusion strategy enables our latent diffusion generator to concentrate perturbations on regions that are universally sensitive across distinct model families, thereby significantly enhancing generalization and robustness against various defense mechanisms.

Overall, JAD establishes a new benchmark for generative black-box adversarial attacks, providing valuable insights and methodologies for future research in cross-architecture adversarial example generation. This approach not only advances theoretical understanding but also marks a meaningful step toward practical, architecture-agnostic adversarial threats in real-world applications.

\bibliographystyle{IEEEtran} 
\bibliography{ref}

\begin{IEEEbiography}
[{\includegraphics[width=1in,height=1.25in,clip,keepaspectratio]{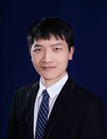}}]{Yang Li}
is an associate professor with the school of automation at Northwestern Polytechnical University, Xi’an, China. After receiving his bachelor’s and doctoral degrees from Northwestern Polytechnical University in 2014 and 2018 respectively, he worked as a research fellow in Sentic Team under Professor Erik Cambria at Nanyang Technological University in Singapore and also was an adjunct research fellow at the A *STAR High-Performance Computing Institute (IHPC). His research goal is to build a trustworthy AI system in the real application, and his research interests are in  Artificial Intelligence Security, etc. He has published several papers on these topics at international conferences and peer-reviewed journals.
\end{IEEEbiography}

\begin{IEEEbiography}[{\includegraphics[width=1in,height=1.25in,clip,keepaspectratio]{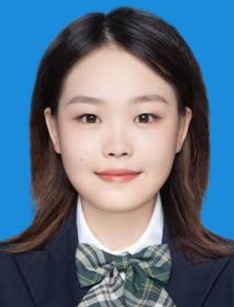}}]{Chenyu Wang}
received her bachelor's degree in Automated Control Engineering from the Xi’an University of Technology, Xi’an, China, in 2023. She is currently working toward the master's degree with the College of Automation, Northwestern Polytechnical University, Xi'an, China. Her research interests include adversarial attack, deep learning and the large model.
\end{IEEEbiography}

\begin{IEEEbiography}[{\includegraphics[width=1in,height=1.25in,clip,keepaspectratio]{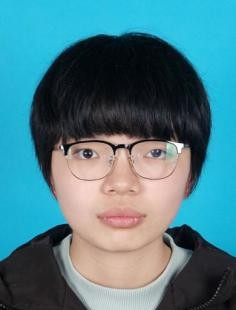}}]{Tingrui Wang}
received her bachelor's degree from the Civil Aviation Flight University of China ,in 2023.She is currently working toward the master's degree with the College of Automation, Northwestern Polytechnical University, Xi'an, China.Her research interests include adversarial attack and deep learning.
\end{IEEEbiography}

\begin{IEEEbiography}[{\includegraphics[width=1in,height=1.25in,clip,keepaspectratio]{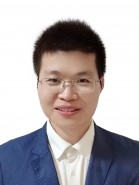}}]{Yongwei Wang}
is a ZJU100 Young Professor at Zhejiang University in Hangzhou, China. He received his Ph.D. in Electrical and Computer Engineering from the University of British Columbia in Vancouver, Canada in 2021, and his M.Sc and and B.Eng degrees from Northwestern Polytechnical University in China in 2017 and 2014, respectively. He worked as a Research Fellow at Nanyang Technological University in Singapore from 2021 to 2023. His research interests include generative AI, AI security and multimedia forensics. He served as reviewers/AC/PCs for many top-tier journals and conferences.
\end{IEEEbiography}

\begin{IEEEbiography}[{\includegraphics[width=1in,height=1.25in,clip,keepaspectratio]{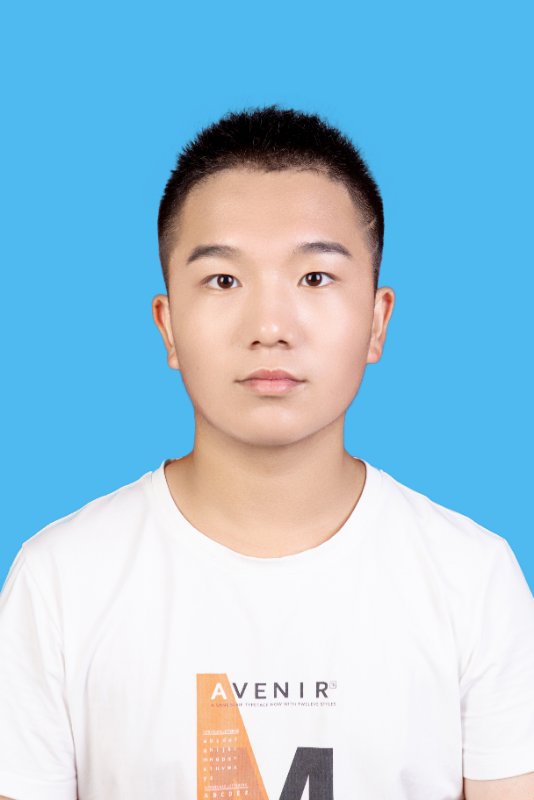}}]{Haonan Li}
received the B.S. degree in automation with South China Agricultural University‌, Guangzhou, China, in 2024. He is currently a graduate researcher with the School of Automation, Northwestern Polytechnical University, Xi'an, China. His research interest is adversarial defense .
\end{IEEEbiography}

\begin{IEEEbiography}
[{\includegraphics[width=1in,height=1.25in,clip,keepaspectratio]{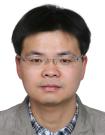}}]{Zhunga Liu}
was born in China. He received the bachelor’s, master’s, and Ph.D. degrees from North
western Polytechnical University (NPU), Xi’an, China, in 2007, 2010, and 2013, respectively. Since 2017, he has been a Professor with the School of Automation, NPU. His current research interests include pattern recognition, information fusion, and belief functions.
\end{IEEEbiography}

\begin{IEEEbiography}
[{\includegraphics[width=1in,height=1.25in,clip,keepaspectratio]{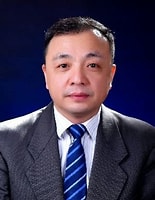}}]{Quan Pan}
was born in China, in 1961. He received the B.S. degree in automatic control from the Huazhong University of Science and Technology, in1982, and the M.S. and Ph.D. degrees in control theory and application from Northwestern Polytechnical University. From 1991 to 1993, he was an Associate Professor at Northwestern Polytechnical University, where he has been a Professor with the Automatic Control Department, since 1997. He has authored 11 books, more than 400 articles. His research interests include information fusion, target tracking and recognition, deep network and machine learning, UAV detection navigation and security control, polarization spectral imaging and image processing, industrial control system information security, commercial password applications, and modern security technologies. He is an Associate Editor of the journal Information Fusion and Modern Weapons Testing Technology.
\end{IEEEbiography}

\vfill

\appendices

% \textbf{Step 4:}Define the student attention alignment degree:  
% \[
% \gamma = \text{sim}(A^G, A^T).  
% \]  
% For any test sample \(x\), the lower bound of the attack success rate is:  
% \[
% \mathbb{P}\left( f_{\text{CNN}}(x+\delta) \neq y \right) \geq g_{\text{CNN}}(\gamma), \quad \mathbb{P}\left( f_{\text{ViT}}(x+\delta) \neq y \right) \geq g_{\text{ViT}}(\gamma),  
% \]  
% where \(g_{\text{CNN}}\) and \(g_{\text{ViT}}\) are monotonically increasing functions of \(\gamma\), satisfying:  
% \[
% \lim_{\gamma \to 1} g_{\text{CNN}}(\gamma) = \mathbb{P}\left( f_{\text{CNN}}(x+\delta^*) \neq y \right), \quad \delta^* = \underset{\delta}{\text{argmax}} \langle \delta, \nabla_x \mathcal{L}_{\text{joint}} \rangle.  
% \]  
% Here, \(\nabla_x \mathcal{L}_{\text{joint}}\) is the gradient direction in the common sensitive regions.  

\end{document}